\documentclass[preprint,10 pt,times]{elsarticle}
\usepackage[top=1in, bottom=1in, left=1.25in, right=1.25in]{geometry}
\usepackage{amsmath,amssymb,bm}
\usepackage{amsthm}
\newtheorem{remark}{Remark}
\usepackage{graphicx}
\usepackage{booktabs}
\usepackage{multirow}
\usepackage{xcolor}
\usepackage{array}
\usepackage{url}
\usepackage{float}
\usepackage{placeins}
\usepackage{stfloats}
\usepackage{tikz}
\usepackage{algorithm}
\usepackage{algorithmic}
\usepackage{textgreek}
\usepackage{hyperref}
\hypersetup{colorlinks=true,linkcolor=blue,citecolor=blue,urlcolor=blue,
	hypertexnames=false}
\newcommand{\rev}[1]{\textcolor{black}{#1}}

\newtheorem{proposition}{Proposition}

\usetikzlibrary{arrows.meta, positioning, calc}

\journal{}
\makeatletter
\def\ps@pprintTitle{%
  \let\@oddhead\@empty
  \let\@evenhead\@empty
  \def\@oddfoot{\hfil\thepage\hfil}%
  \let\@evenfoot\@oddfoot}
\makeatother

\begin{document}

\begin{frontmatter}

\title{When Does Perturbation Consistency Add Information Beyond Softmax Confidence? Theory and Evidence for the Decision Consistency Framework}

\author[DUIET]{Amit Kalita}
\ead{contactamitkalita@gmail.com}

\author[DUIET]{Moirangthem Tiken Singh\corref{cor1}}
\ead{tiken.m@dibru.ac.in}
\cortext[cor1]{Corresponding author.}

\address[DUIET]{Department of Computer Science and Engineering, Dibrugarh University Institute of Engineering and Technology, Dibrugarh University, Dibrugarh, Assam, India}

\begin{abstract}
\rev{Perturbation consistency scores judge whether a prediction survives small input changes. The Decision Consistency Framework (DCF) scores each image by label stability (LS) and confidence deviation (CD) under six mild operators, combines them into a consistency score (CS), and summarises a model by the gap between accuracy and mean consistency (CAG). We ask when these scores carry information that the clean softmax output lacks. Across 201 evaluations with three seeds per configuration, CS detected misclassifications on clean data less well than the maximum softmax probability (MSP) in every model family. For frozen CNNs the deficit was 0.057 in area under the ROC curve (AUROC; 95\% CI 0.043 to 0.073), at 8 to 15 times the compute. A label flip requires the logit margin to be at most the square root of two times the logit displacement, and at a fixed margin the displacement predicted errors at AUROC 0.515. The same inequality gives a label-free split-conformal no-flip certificate, never violated in 84 runs. Consistency did add information where confidence is misled: under noise corruptions of CIFAR-10-C it beat the margin, and under projected-gradient attacks it separated clean from attacked inputs with AUROC 0.83 against 0.37 for MSP at a budget of 2/255, as pre-specified. The gain under noise did not transfer to unseen shift families. The large negative CAG on OCTMNIST persists under full fine-tuning and in-domain training and is partly a class-prior effect: correcting the prior recovers about 40\% of the gap.}
\end{abstract}

\begin{keyword}
decision consistency \sep
misclassification detection \sep
conformal prediction \sep
distribution shift \sep
label shift \sep
adversarial detection
\end{keyword}

\end{frontmatter}

\section{Introduction}
\label{sec:intro}

Accuracy on a held-out test set says how often a classifier is right on average. It does not say which of its predictions should be trusted. In deployment the second question matters more. A radiology triage tool or an inspection camera has to decide, image by image, whether to act on the output or pass the case to a person~\cite{amodei2016concrete,tjoa2020survey}. The usual answer is a confidence score. The maximum softmax probability (MSP)~\cite{hendrycks2017baseline}, the logit margin, the predictive entropy and the energy score~\cite{liu2020energy} are all computed from one forward pass and are difficult to outperform for misclassification detection~\cite{jaeger2023call,galil2023what,cattelan2024fix}.

\rev{These scores read only the clean output. They cannot show whether a prediction would survive a small, harmless change to the input. A second family of scores addresses this gap. Test-time augmentation agreement~\cite{ayhan2018tta,shanmugam2021tta}, the flip probabilities of ImageNet-P~\cite{hendrycks2019benchmarking} and feature squeezing for adversarial detection~\cite{xu2018feature} all follow this idea. The Decision Consistency Framework (DCF) is one instance of it. For each test image, DCF applies six mild operators: rotation, brightness, Gaussian noise, contrast, blur and JPEG compression. It records how often the predicted label is kept (label stability, LS) and how much the top-class confidence moves (confidence deviation, CD). It then combines the two into a consistency score (CS). At the model level, DCF reports the consistency-accuracy gap (CAG), the difference between clean accuracy and mean CS.}

\rev{The appeal of such scores is intuitive. A prediction that flips under a rotation of two degrees looks fragile, and one that never moves looks dependable. The cost is also clear, because each operator needs one extra forward pass. That cost is justified only if consistency carries information that the clean output does not already contain. Arguments for such scores usually rest on Lipschitz-style bounds and on flip frequencies. They seldom rest on a comparison with the confidence the scores are meant to improve on. The central question is therefore:}

\begin{quote}
\emph{When does perturbation consistency add information about the correctness of a prediction beyond what the clean softmax output already provides?}
\end{quote}

\rev{We answer it with a theoretical argument and a pre-specified set of experiments. The answer is mostly negative, with two exceptions. The contributions follow the order of the argument.}

\begin{enumerate}
\item \textbf{\rev{A mechanism for redundancy.}} \rev{We prove that a label flip under any operator implies that the clean logit margin is at most the square root of two times the norm of the logit displacement. It follows that LS and CD can add information only through the dependence of the displacement on correctness at a fixed clean output. Measured directly within margin bins, that dependence is negligible on clean data.}

\item \textbf{\rev{A finite-sample no-flip certificate.}} \rev{The same inequality yields a split-conformal rule that certifies an image from its clean margin alone. The calibration set needs no labels. The probability that an image is certified and still flips is bounded by a chosen level, and this bound held in every run at every level tested. Unlike a Lipschitz flip-probability bound, which is vacuous in this setting, the rule needs neither a Lipschitz constant nor a shared perturbation budget.}

\item \textbf{\rev{A broad comparison against confidence baselines.}} \rev{The mechanism predicts that consistency is redundant on clean data, and the comparison tests this prediction. It covers frozen and natively trained CNNs, vision transformers, a model zoo and last-layer ensembles, evaluated on several natural-image and medical-image datasets. CS was below MSP on average for clean error detection in every model family. LS alone was close to chance among high-confidence predictions. No reformulation of CD and no choice of the weight between LS and CD closed the gap.}

\item \textbf{\rev{A boundary condition.}} \rev{The same mechanism indicates where redundancy should fail, namely where the clean output is itself misled. Consistency signals did add information in two such settings. The first is noise-type corruptions of natural images, where the gain grew with MSP over-confidence. The second is adversarial inputs, where a pre-specified prediction was confirmed in every run. The gain does not transfer to shift families not seen in training. An attacker who penalises instability can evade the detector, at the price of a much lower attack success rate.}

\item \textbf{\rev{A partial explanation of ``stable but wrong'' predictions.}} \rev{The last finding concerns CAG. Models trained on a retinal imaging task had a large negative CAG, which means that many wrong predictions were stable. The pattern survives longer head training, full fine-tuning and training from scratch on the target domain, and at matched accuracy the in-domain model is no better. Part of it is a class-prior shift between the official splits. Correcting the prior recovers part of the gap, and inducing the same kind of shift on a natural-image benchmark reproduces it.}
\end{enumerate}

\rev{These findings suggest a rule for practice. Use MSP or the margin for clean error detection. Use the conformal certificate when a cheap guarantee against small perturbations is needed. Reserve consistency checks for settings where confidence itself is known to be corrupted.}

The rest of the paper is organised as follows. Section~\ref{sec:related} reviews related work. Section~\ref{sec:framework} defines the DCF quantities and states the theory. Section~\ref{sec:protocol} describes the protocol, including the perturbation audit and the pre-specified predictions. Section~\ref{sec:results} reports the results, Section~\ref{sec:discussion} discusses their implications and limitations, and Section~\ref{sec:conclusion} concludes.

\section{Related Work}
\label{sec:related}

\rev{This section follows the structure of the argument. It first reviews the single-pass confidence scores that serve as baselines. It then turns to perturbation-based scores, to the behaviour of uncertainty estimates under shift, and to the margin and conformal tools that underlie the theory. It ends with adversarial detection and label shift, two settings in which confidence can be misleading.}

\subsection{Misclassification detection from a single forward pass}

\rev{Hendrycks and Gimpel~\cite{hendrycks2017baseline} showed that MSP separates correct from incorrect predictions well above chance. It has remained a strong baseline since then. Later scores reuse the same logits. They include the margin between the two largest logits, the predictive entropy, the energy score~\cite{liu2020energy}, the maximum logit~\cite{hendrycks2022scaling} and DOCTOR~\cite{granese2021doctor}. Temperature scaling~\cite{guo2017calibration} improves calibration without changing the ranking of MSP. Careful evaluations found that elaborate alternatives rarely outperform these baselines for failure detection~\cite{jaeger2023call,galil2023what}. They also found that the remaining weaknesses of MSP can often be fixed post hoc~\cite{cattelan2024fix,gomes2024relative}.}

\rev{Other families of methods address related problems. Feature-space scores such as the Mahalanobis distance~\cite{lee2018simple} and deterministic uncertainty methods~\cite{mukhoti2023ddu} target out-of-distribution (OOD) inputs~\cite{yang2024generalized} more than misclassification. Training-time methods such as OpenMix~\cite{zhu2023openmix} change the model and fall outside the post-hoc setting considered here. Recent work applies failure detection to vision-language models~\cite{zeng2025fsmisd} and to concept-based explanations of errors~\cite{nguyen2025concepts}. A benchmark of confidence scores on medical image classification under distribution shift found that no single score dominated across settings~\cite{steinmetz2026medfail}. The single-pass scores therefore form the baselines that any perturbation-based score must outperform.}

\subsection{Perturbation consistency and test-time augmentation}

\rev{Perturbation-based scores start from a different idea. They ask whether a prediction survives a change to the input. Averaging predictions over augmented views, known as test-time augmentation (TTA), can improve accuracy and uncertainty estimates~\cite{ayhan2018tta,shanmugam2021tta}. ImageNet-C and ImageNet-P~\cite{hendrycks2019benchmarking} measure average-case robustness to corruptions and the probability that a prediction flips along a perturbation sequence. The motivation is that convolutional networks can change their output under shifts of a few pixels or small rotations~\cite{azulay2019why,engstrom2019rotation}.}

\rev{Consistency under perturbation is also used in other roles. It serves as a training signal~\cite{miyato2019virtual,hendrycks2020augmixsimpledataprocessing,zhang2022memo} and as an objective for test-time adaptation~\cite{wang2021tent}. For detection, feature squeezing~\cite{xu2018feature} flags adversarial inputs whose prediction changes under bit-depth reduction or smoothing. DCF belongs to this family. It differs from TTA in using the agreement and the confidence change as a reliability score rather than averaging the predictions. TTA and the single-pass scores are therefore both relevant comparisons.}

\subsection{Uncertainty under distribution shift}

\rev{Reliability scores are most needed when the data change, and uncertainty estimates degrade under shift~\cite{ovadia2019trust}. Deep ensembles~\cite{lakshminarayanan2017simple} and Monte Carlo dropout~\cite{gal2016dropout} remain common references. At the model level, accuracy under shift is often linearly related to in-distribution accuracy~\cite{miller2021accuracy,taori2020measuring}. Agreement between models, or the average confidence, can predict accuracy on unlabelled shifted data~\cite{jiang2022disagreement,baek2022agreement,garg2022atc,guillory2021doc}. A clean-data score that is claimed to predict behaviour under shift therefore has to outperform clean accuracy itself as a predictor. A correlation with robustness is not sufficient.}

\subsection{Margins, Lipschitz bounds and conformal guarantees}

\rev{Guarantees about prediction changes rest on the logit margin. Lipschitz-margin arguments give deterministic certificates. If the logit margin exceeds a constant times the product of the Lipschitz constant and the perturbation budget, no perturbation within the budget can flip the label~\cite{hein2017formal,tsuzuku2018lipschitz,weng2018clever,fazlyab2023certified}. Randomised smoothing gives probabilistic certificates for a smoothed classifier~\cite{cohen2019certified}. Both approaches need either a Lipschitz constant or a change to the classifier.}

\rev{Conformal prediction offers a different route. It gives distribution-free, finite-sample guarantees under exchangeability~\cite{angelopoulos2021uncertainty,angelopoulos2023gentle,lei2018distribution}, and its behaviour when exchangeability fails is well studied~\cite{tibshirani2019covariate,barber2023beyond}. The certificate proposed here combines the two approaches. The margin inequality is deterministic. The threshold is set by a conformal quantile of measured displacements, so no Lipschitz constant is needed.}

\subsection{Adversarial detection and adaptive attacks}

\rev{Adversarial inputs are a setting in which confidence is deliberately misled. Projected gradient descent (PGD)~\cite{madry2017towards} is the standard first-order attack. Ensembles of parameter-free attacks are used when robustness itself is the claim~\cite{croce2020autoattack}. Detectors that look convincing against fixed attacks often fail once the attacker knows the detector~\cite{carlini2017evaluating,athalye2018obfuscated,tramer2020adaptive}. Non-differentiable components such as JPEG compression must be attacked through approximations such as the backward pass differentiable approximation (BPDA)~\cite{athalye2018obfuscated}. Evidence for a detector therefore requires an adaptive attack as well as a standard one.}

\subsection{Label shift}

\rev{Class priors are a second source of misleading confidence. When class priors differ between training and test data while the class-conditional distributions do not, a classifier's posteriors are biased towards the training prior. The prior can be re-estimated without labels by expectation-maximisation~\cite{saerens2002adjusting} or by black-box shift estimation~\cite{lipton2018bbse}. Calibrated maximum likelihood is a strong choice~\cite{alexandari2020mlls,garg2020unified}. Class imbalance more generally leaves minority classes with low recall~\cite{zhang2023longtail}. This literature makes label shift a candidate explanation for confident but wrong predictions.}

\section{Decision consistency: definitions and theory}
\label{sec:framework}

\subsection{Notation}
\label{subsec:notation}

Let $z(\mathbf{x})\in\mathbb{R}^{J}$ be the logits of a classifier with $J$ classes for an input $\mathbf{x}$, $p(\mathbf{x})=\mathrm{softmax}(z(\mathbf{x}))$ the class probabilities, $\hat{y}(\mathbf{x})=\arg\max_j z_j(\mathbf{x})$ the prediction and $c(\mathbf{x})=\max_j p_j(\mathbf{x})$ the MSP. The clean logit margin is
\begin{equation}
\gamma(\mathbf{x}) = z_{\hat{y}(\mathbf{x})}(\mathbf{x}) - \max_{j\neq\hat{y}(\mathbf{x})} z_j(\mathbf{x}) \;\ge 0 .
\label{eq:margin}
\end{equation}
\rev{A perturbation set $\mathcal{T}=\{T_1,\dots,T_K\}$ contains $K$ operators. An operator may be random, as Gaussian noise is. In that case its randomness is drawn independently for every image.} For image $\mathbf{x}_i$ we write \rev{$\hat{y}_i=\hat{y}(\mathbf{x}_i)$, $c_i=c(\mathbf{x}_i)$, $\gamma_i=\gamma(\mathbf{x}_i)$,} $\mathbf{x}_i^{(k)}=T_k(\mathbf{x}_i)$, $\hat{y}_i^{(k)}=\hat{y}(\mathbf{x}_i^{(k)})$, $c_i^{(k)}=c(\mathbf{x}_i^{(k)})$, and
\begin{equation}
\Delta_{i,k} = z(\mathbf{x}_i^{(k)}) - z(\mathbf{x}_i) \in \mathbb{R}^{J}
\label{eq:disp}
\end{equation}
for the logit displacement caused by operator $k$. \rev{We say that operator $k$ \emph{flips} image $i$ if $\hat{y}_i^{(k)}\neq\hat{y}_i$. With ties, a flip means that some class $j\neq\hat{y}_i$ satisfies $z_j(\mathbf{x}_i^{(k)})\ge z_{\hat{y}_i}(\mathbf{x}_i^{(k)})$. Correctness is $Z_i=\mathbf{1}[\hat{y}_i=y_i]$, where $\mathbf{1}[\cdot]$ is the indicator function and $y_i$ is the true label.}

\subsection{The DCF quantities}
\label{subsec:metrics}

\rev{DCF summarises the behaviour of a prediction under the $K$ operators with three image-level quantities and one model-level quantity. Label stability (LS) is the fraction of operators that keep the clean prediction,}
\begin{equation}
S_i = \frac{1}{K}\sum_{k=1}^{K}\mathbf{1}\!\left[\hat{y}_i^{(k)}=\hat{y}_i\right]\in[0,1].
\label{eq:ls}
\end{equation}
\rev{Confidence deviation (CD) measures the complementary effect on confidence. It is the mean absolute change of the top-class confidence,}
\begin{equation}
D_i = \frac{1}{K}\sum_{k=1}^{K}\bigl|c_i - c_i^{(k)}\bigr|\in[0,1].
\label{eq:cd}
\end{equation}
\rev{Equation~(\ref{eq:cd}) compares maximum probabilities even when the winning class changes. A prediction can therefore flip from one class to another with the same confidence and still have $D_i\approx 0$. Equation~(\ref{eq:cd}) is the reference definition used throughout.}

\rev{The consistency score (CS) combines the two quantities as a convex combination,}
\begin{equation}
C_i = \alpha S_i + (1-\alpha)(1-D_i),\qquad \alpha\in[0,1],
\label{eq:cs}
\end{equation}
\rev{with $\alpha=0.5$ unless stated otherwise. A single value of $C_i$ does not determine $S_i$ and $D_i$, because many pairs give the same score. LS and CD are therefore always reported alongside CS, and CS is read only as a ranking score.}

\rev{At the model level, the consistency-accuracy gap (CAG) is}
\begin{equation}
\mathrm{CAG} = \mathrm{Acc} - \bar{C},\qquad \bar{C}=\frac{1}{N}\sum_{i=1}^{N}C_i ,
\label{eq:cag}
\end{equation}
\rev{where Acc is the accuracy on the full test set and $\bar{C}$ is averaged over $N$ evaluation images. Accuracy measures agreement with the true label, while CS measures agreement of the model with itself. The difference of the two is therefore not a calibration measure, and it carries no guarantee. CAG is used only as a descriptive flag. For a tolerance $\delta$, we call a model \emph{accuracy-dominant} if $\mathrm{CAG}>\delta$, \emph{consistency-dominant} if $\mathrm{CAG}<-\delta$ and \emph{balanced} otherwise. The label ``balanced'' makes no calibration claim. Calibration refers to the agreement between confidence and accuracy~\cite{guo2017calibration}.}

\subsection{What a flip implies about the margin}
\label{subsec:flip}

\rev{The first result connects a label flip to the clean margin. It holds for every image and every operator, and it requires no assumption about the classifier.}

\begin{proposition}[Flip implication]
\label{prop:flip}
For every image $\mathbf{x}_i$ and operator $T_k$, if $T_k$ flips $\mathbf{x}_i$ then
\begin{equation}
\gamma(\mathbf{x}_i)\;\le\; \bigl|\Delta_{i,k,\hat{y}_i}\bigr| + \max_{j\neq\hat{y}_i}\bigl|\Delta_{i,k,j}\bigr| \;\le\; \sqrt{2}\,\bigl\|\Delta_{i,k}\bigr\|_2 .
\label{eq:flip}
\end{equation}
Consequently $1-S_i \le \frac{1}{K}\sum_{k}\mathbf{1}\bigl[\gamma(\mathbf{x}_i)\le\sqrt{2}\|\Delta_{i,k}\|_2\bigr]$, and $S_i=1$ whenever $\gamma(\mathbf{x}_i)>\sqrt{2}\max_k\|\Delta_{i,k}\|_2$.
\end{proposition}

\begin{proof}
Write $\hat{y}=\hat{y}_i$, $z=z(\mathbf{x}_i)$, $z'=z(\mathbf{x}_i^{(k)})$ and $\Delta=z'-z$. If $T_k$ flips $\mathbf{x}_i$, there is a class $j^\ast\neq\hat{y}$ with $z'_{j^\ast}\ge z'_{\hat{y}}$. By the definition of the margin, $z_{j^\ast}\le\max_{j\neq\hat{y}}z_j = z_{\hat{y}}-\gamma(\mathbf{x}_i)$. Hence
\[
0 \;\ge\; z'_{\hat{y}}-z'_{j^\ast} \;=\; (z_{\hat{y}}-z_{j^\ast}) + (\Delta_{\hat{y}}-\Delta_{j^\ast}) \;\ge\; \gamma(\mathbf{x}_i) + \Delta_{\hat{y}}-\Delta_{j^\ast},
\]
so $\gamma(\mathbf{x}_i)\le\Delta_{j^\ast}-\Delta_{\hat{y}}\le|\Delta_{\hat{y}}|+|\Delta_{j^\ast}|\le|\Delta_{\hat{y}}|+\max_{j\neq\hat{y}}|\Delta_j|$, which is the first inequality. For the second, let $a=|\Delta_{\hat{y}}|$ and $b=\max_{j\neq\hat{y}}|\Delta_j|=|\Delta_{j'}|$ for some $j'\neq\hat{y}$. By the Cauchy-Schwarz inequality, $a+b\le\sqrt{2}\sqrt{a^2+b^2}$, and $a^2+b^2=\Delta_{\hat{y}}^2+\Delta_{j'}^2\le\|\Delta\|_2^2$ because $\hat{y}\neq j'$ are two distinct coordinates. Taking square roots gives $a+b\le\sqrt{2}\|\Delta\|_2$. The statement about $S_i$ follows by applying the implication to each operator and averaging the indicators over $k$.
\end{proof}

\rev{Bounding the change of the winning logit and of the runner-up separately by $\|\Delta_{i,k}\|_2$ gives the constant 2. Combining the two coordinates as in the proof gives $\sqrt{2}$, which is tight (take $\Delta$ supported on $\hat{y}$ and one competitor with equal and opposite entries). If the logit map is $L$-Lipschitz and $\|\mathbf{x}_i^{(k)}-\mathbf{x}_i\|_2\le\varepsilon$, then $\|\Delta_{i,k}\|_2\le L\varepsilon$ and Eq.~(\ref{eq:flip}) gives the familiar certificate $\gamma>\sqrt{2}L\varepsilon\Rightarrow$ no flip~\cite{tsuzuku2018lipschitz}.}

\rev{A probabilistic version of this certificate is not informative in this setting.}

\begin{remark}[\rev{Limitations of a probabilistic Lipschitz bound}]
\label{rem:withdraw}
A bound of the form $P(\hat{y}_i^{(k)}\neq\hat{y}_i)\le\min(1,2L\varepsilon/\gamma_i)$, obtained from Markov's inequality, is vacuous: when $\gamma_i\le 2L\varepsilon$ it equals one, and when $\gamma_i>2L\varepsilon$ the deterministic certificate already gives zero. It also needs one budget $\varepsilon$ shared by all operators, \rev{which operators of different kinds do not generally share}. \rev{In addition, checking that images with $\gamma_i>\sqrt{2}\|\Delta_{i,k}\|_2$ do not flip, with $\Delta_{i,k}$ measured on the same pair, cannot validate the bound, because Proposition~\ref{prop:flip} guarantees the outcome. This check serves only as an implementation test. The conformal certificate avoids these limitations.}
\end{remark}

\subsection{A split-conformal no-flip certificate}
\label{subsec:conformal}

\rev{Proposition~\ref{prop:flip} needs the displacement of the image being certified, and that displacement is available only after the perturbed image has been processed. A more practical rule uses only the clean margin. Its threshold is calibrated on the worst-case scaled displacement of an image,}
\begin{equation}
R(\mathbf{x}_i) = \sqrt{2}\,\max_{k}\,\bigl\|\Delta_{i,k}\bigr\|_2 ,
\label{eq:score}
\end{equation}
which depends on the image and on the draw of the random operators, but not on its label.

\begin{proposition}[Conformal no-flip certificate]
\label{prop:conformal}
Let $\mathbf{x}_1,\dots,\mathbf{x}_n$ be calibration images and $\mathbf{x}_{n+1}$ a test image such that $R(\mathbf{x}_1),\dots,R(\mathbf{x}_{n+1})$ are exchangeable. For $\beta\in(0,1)$ let $m=\lceil(n+1)(1-\beta)\rceil$ and let $\hat{q}$ be the $m$-th smallest of $R(\mathbf{x}_1),\dots,R(\mathbf{x}_n)$, with $\hat{q}=\infty$ if $m>n$. Certify $\mathbf{x}_{n+1}$ if $\gamma(\mathbf{x}_{n+1})>\hat{q}$. Then
\begin{equation}
P\bigl(\mathbf{x}_{n+1}\text{ is certified and some operator flips it}\bigr)\;\le\;P\bigl(R(\mathbf{x}_{n+1})>\hat{q}\bigr)\;\le\;\beta ,
\label{eq:conformal}
\end{equation}
where the probability is over the calibration images, the test image and the operator randomness.
\end{proposition}

\begin{proof}
If $\mathbf{x}_{n+1}$ is certified and some operator $k$ flips it, Proposition~\ref{prop:flip} gives $\hat{q}<\gamma(\mathbf{x}_{n+1})\le\sqrt{2}\|\Delta_{n+1,k}\|_2\le R(\mathbf{x}_{n+1})$. The event on the left of Eq.~(\ref{eq:conformal}) is therefore contained in $\{R(\mathbf{x}_{n+1})>\hat{q}\}$, which proves the first inequality. For the second, by exchangeability the rank of $R(\mathbf{x}_{n+1})$ among the $n+1$ values is uniformly distributed on $\{1,\dots,n+1\}$ when there are no ties, and ties only make the event $R(\mathbf{x}_{n+1})>\hat{q}$ less likely. The event $R(\mathbf{x}_{n+1})>\hat{q}$ requires $R(\mathbf{x}_{n+1})$ to be larger than at least $m$ of the calibration values, so its rank is at least $m+1$, which has probability at most $(n+1-m)/(n+1)\le\beta$ by the choice of $m$~\cite{lei2018distribution,angelopoulos2023gentle}. If $m>n$, $\hat{q}=\infty$ and nothing is certified.
\end{proof}

\rev{The certificate has three properties. First, the threshold is fixed on separate data, and at test time only the clean margin is used, so the perturbed versions of the test image are never computed. Second, the certificate needs no labels, because $R$ is computed from logits alone, and any unlabelled sample from the deployment distribution can serve as the calibration set. Third, the guarantee is marginal and depends on exchangeability. Under distribution shift the calibration quantile may be too small.} A per-sample flip bound can be built in the same way from the empirical distribution of $\sqrt{2}\|\Delta_{i,k}\|_2$ on the calibration split, and the Dvoretzky-Kiefer-Wolfowitz (DKW) inequality~\cite{massart1990tight} makes it uniform in finite samples.

\subsection{When can consistency add information?}
\label{subsec:redundancy}

\rev{The flip implication also bears on the central question of whether LS and CD can add information beyond the clean output. Proposition~\ref{prop:flip} says that a flip depends on the margin and on the size of the displacement. The clean output already contains the margin. Consistency can therefore add information only through the displacement, which the next proposition makes precise.}

\begin{proposition}[Redundancy condition]
\label{prop:redundancy}
\rev{Write $Z$, $S$ and $D$ for the correctness, LS and CD of an image drawn at random, and $\Delta_k$ for the displacement caused by operator $k$ (the image index is dropped).} LS and CD are deterministic functions of the clean logits $z(\mathbf{x})$ and the displacements $(\Delta_{1},\dots,\Delta_{K})$. If the displacements are conditionally independent of correctness given the clean logits, $(\Delta_{1},\dots,\Delta_{K})\perp Z\mid z(\mathbf{x})$, then $P(Z=1\mid z,S,D)=P(Z=1\mid z)$. In that case no error-detection score built from $(z,S,D)$ has a higher area under the ROC curve (AUROC) than the posterior $P(Z=1\mid z)$, which uses the clean logits alone.
\end{proposition}

\begin{proof}
$(S,D)$ is a function of $(z,\Delta_1,\dots,\Delta_K)$, so conditional independence of the displacements from $Z$ given $z$ implies $(S,D)\perp Z\mid z$, and hence $P(Z=1\mid z,S,D)=P(Z=1\mid z)$. For any feature vector, the ROC curve of the posterior probability of the positive class dominates the ROC curve of every other score computed from the same features (the Neyman-Pearson lemma applied to the likelihood ratio, which is a monotone function of the posterior). Since the posterior given $(z,S,D)$ equals the posterior given $z$, the best achievable ROC curve is the same with or without $(S,D)$.
\end{proof}

\rev{Proposition~\ref{prop:redundancy} makes the question of added information testable. Consistency can only help if two images with the same clean logits, one classified correctly and one not, tend to move differently under the operators. The cleanest test therefore fixes the margin and asks whether the displacement still separates correct from incorrect predictions. The condition is expected to fail when the clean output itself is corrupted, for example by input noise or by an adversarial perturbation designed to raise confidence. In such cases the displacement caused by a smoothing operator can reveal what the clean logits hide~\cite{xu2018feature}.}

\subsection{Evaluation pipeline}
\label{subsec:pipeline}

Fig.~\ref{fig:framework} \rev{shows the pipeline and where the three theoretical results enter.} Computing LS, CD and CS for $N$ images costs $N(K+1)$ forward passes. The certificate costs one forward pass per test image once $\hat{q}$ has been computed on the calibration split.

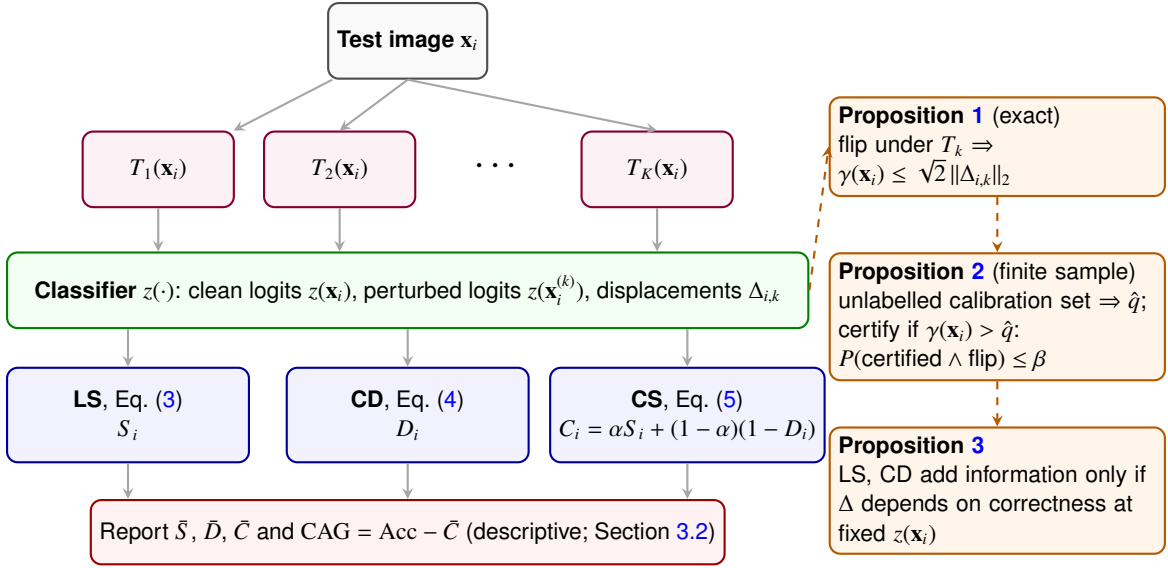
\begin{figure*}[h!]
\centering
\begin{tikzpicture}[
    >=stealth, font=\sffamily\small,
    box/.style={draw=black!70, rounded corners, align=center, fill=gray!5, thick, minimum height=1.0cm},
    perturb/.style={draw=purple!70!black, fill=purple!5, rounded corners, minimum width=2.0cm, minimum height=1.0cm, align=center, thick},
    infer/.style={draw=green!50!black, fill=green!5, rounded corners, minimum width=10.6cm, minimum height=1.0cm, align=center, thick},
    metric/.style={draw=blue!60!black, fill=blue!5, rounded corners, minimum width=3.2cm, minimum height=1.25cm, align=center, thick},
    theory/.style={draw=orange!70!black, fill=orange!8, rounded corners, align=left, thick, text width=4.2cm},
    final/.style={draw=red!60!black, fill=red!5, rounded corners, minimum height=0.85cm, align=center, thick},
    arrow/.style={->, draw=gray!70, thick}]
  \node[box] (input) at (-0.9, 3.3) {\textbf{Test image} $\mathbf{x}_i$};
  \node[perturb] (p1) at (-4.2, 1.6) {$T_1(\mathbf{x}_i)$};
  \node[perturb] (p2) at (-1.8, 1.6) {$T_2(\mathbf{x}_i)$};
  \node (dots) at (0.3, 1.6) {\Large $\cdots$};
  \node[perturb] (pk) at (2.4, 1.6) {$T_K(\mathbf{x}_i)$};
  \node[infer] (infer) at (-0.9, 0.0) {\textbf{Classifier} $z(\cdot)$: clean logits $z(\mathbf{x}_i)$, perturbed logits $z(\mathbf{x}_i^{(k)})$, displacements $\Delta_{i,k}$};
  \node[metric] (ls) at (-4.6, -1.65) {\textbf{LS}, Eq.~(\ref{eq:ls})\\ $S_i$};
  \node[metric] (cd) at (-0.9, -1.65) {\textbf{CD}, Eq.~(\ref{eq:cd})\\ $D_i$};
  \node[metric] (cs) at (2.8, -1.65) {\textbf{CS}, Eq.~(\ref{eq:cs})\\ $C_i=\alpha S_i+(1-\alpha)(1-D_i)$};
  \node[final] (cag) at (-0.9, -3.2) {Report $\bar{S}$, $\bar{D}$, $\bar{C}$ and $\mathrm{CAG}=\mathrm{Acc}-\bar{C}$ (descriptive; Section~\ref{subsec:metrics})};
  \node[theory] (prop1) at (6.9, 1.9) {\textbf{Proposition~\ref{prop:flip}} (exact)\\ flip under $T_k$ $\Rightarrow$\\ $\gamma(\mathbf{x}_i)\le\sqrt{2}\,\|\Delta_{i,k}\|_2$};
  \node[theory] (prop2) at (6.9, -0.35) {\textbf{Proposition~\ref{prop:conformal}} (finite sample)\\ unlabelled calibration set $\Rightarrow\hat{q}$;\\ certify if $\gamma(\mathbf{x}_i)>\hat{q}$:\\ $P(\text{certified}\wedge\text{flip})\le\beta$};
  \node[theory] (prop3) at (6.9, -2.65) {\textbf{Proposition~\ref{prop:redundancy}}\\ LS, CD add information only if $\Delta$ depends on correctness at fixed $z(\mathbf{x}_i)$};
  \draw[arrow] (input) -- (p1); \draw[arrow] (input.south) -- (p2.north); \draw[arrow] (input.south) -- (pk.north);
  \draw[arrow] (p1) -- (p1 |- infer.north); \draw[arrow] (p2) -- (p2 |- infer.north); \draw[arrow] (pk) -- (pk |- infer.north);
  \draw[arrow] (ls |- infer.south) -- (ls); \draw[arrow] (infer.south) -- (cd); \draw[arrow] (cs |- infer.south) -- (cs);
  \draw[arrow] (ls) -- (ls |- cag.north); \draw[arrow] (cd) -- (cag); \draw[arrow] (cs) -- (cs |- cag.north);
  \draw[arrow, dashed, draw=orange!70!black] (infer.east) -- (prop1.west);
  \draw[arrow, dashed, draw=orange!70!black] (prop1) -- (prop2);
  \draw[arrow, dashed, draw=orange!70!black] (prop2) -- (prop3);
\end{tikzpicture}
\caption{Evaluation pipeline and where the theory enters. The left column computes the DCF quantities from $K+1$ forward passes per image. The right column shows the three results of Section~\ref{sec:framework}: the exact flip implication (Proposition~\ref{prop:flip}), the conformal certificate that needs only the clean margin of a test image and an unlabelled calibration set (Proposition~\ref{prop:conformal}), and the condition under which LS and CD can carry information beyond the clean logits (Proposition~\ref{prop:redundancy}).}
\label{fig:framework}
\end{figure*}

\section{Experimental protocol}
\label{sec:protocol}

\subsection{Datasets}
\label{subsec:datasets}

\rev{The experiments use four image classification datasets (Table~\ref{tab:datasets}). They are CIFAR-10~\cite{krizhevsky2009learning}, STL-10~\cite{coates2011stl10}, the Intel Image Classification scene dataset~\cite{intelimage} and OCTMNIST~\cite{yang2021medmnist,kermany2018oct}. The first three contain natural images, and OCTMNIST contains retinal optical coherence tomography images. For each run we draw an evaluation set of $N=1{,}000$ images and a disjoint calibration set of $1{,}000$ images. On CIFAR-10, STL-10 and Intel-Scene, both sets come from the official test split. OCTMNIST has only 1,000 test images. The whole test split is therefore the evaluation set, and the calibration set is drawn from the official validation split. The OCTMNIST training and validation splits share one class prior, while the test split is balanced. Accuracy and expected calibration error (ECE, 15 equal-width bins) are computed on the full test split. DCF quantities are computed on the evaluation set.}

\begin{table}[h!]
\centering
\caption{Datasets. Images are resized to $224\times224$ for the ImageNet backbones after the perturbation is applied at native resolution. OCTMNIST class priors are given as (choroidal neovascularization, diabetic macular edema, drusen, normal).}
\label{tab:datasets}
\small
\begin{tabular}{@{}lccl@{}}
\toprule
Dataset & Resolution & Classes & Split used for evaluation / calibration \\
\midrule
CIFAR-10 & $32\times32$ & 10 & test / test (disjoint) \\
STL-10 & $96\times96$ & 10 & test / test (disjoint) \\
Intel-Scene & $150\times150$ & 6 & test / test (disjoint) \\
OCTMNIST & $28\times28$ & 4 & test (balanced) / validation, prior (0.344, 0.105, 0.080, 0.472) \\
\bottomrule
\end{tabular}
\end{table}

\subsection{Models and adaptation}
\label{subsec:models}

\rev{We evaluate five groups of models. Together they cover frozen and fully trained classifiers, convolutional and transformer architectures, and single models and ensembles.}

\rev{The first group contains frozen CNNs: ResNet-18, ResNet-50~\cite{he2016deep}, VGG-16~\cite{simonyan2015very} and MobileNetV2~\cite{sandler2018mobilenetv2}. Each uses frozen ImageNet weights and one linear head trained on the target data (SGD, learning rate $10^{-3}$, momentum 0.9, batch size 32, 5 epochs). The same head produces accuracy, ECE and all DCF quantities. This gives 16 configurations, CNN-001 to CNN-016.}

\rev{The second group contains frozen vision transformers: ViT-B/16~\cite{dosovitskiy2020image}, DeiT-B~\cite{touvron2021deit} and Swin-B~\cite{liu2021swin}. Each is trained with a linear head by exactly the same recipe on each target dataset, which gives 12 configurations. The transformers are always evaluated with a trained head, because their ImageNet output space does not match the target classes.}

\rev{The third group contains fourteen publicly released CIFAR-10 networks that were trained end to end at $32\times32$. They are ResNet-20/32/44/56, VGG-11/13/16/19 with batch normalisation, MobileNetV2 $\times0.5$ and $\times1.4$, ShuffleNetV2 $\times0.5$ and $\times2.0$, and RepVGG-A0 and A2}\footnote{Weights from the \texttt{chenyaofo/pytorch-cifar-models} repository.}\rev{. These networks contrast with the frozen backbones, which use only a short linear probe. For them the seed changes the evaluation images and the operator noise, not the training.}

\rev{The fourth group is a 13-model zoo for model-level questions. It contains the four frozen CNNs and three vision transformers, plus ResNet-34, DenseNet-121, EfficientNet-B0, ConvNeXt-T~\cite{liu2022convnext}, RegNetY-800MF and ShuffleNetV2. All use linear heads trained on CIFAR-10 by the same recipe.}

\rev{The fifth group consists of adaptation controls on OCTMNIST and Intel-Scene. Linear probing and full fine-tuning can produce different features and failure modes~\cite{kumar2022finetuning,kirichenko2023lastlayer}. The controls therefore vary the training regime. They comprise three conditions. The first is the frozen head of each of the four CNNs, trained for up to 30 epochs with early stopping on a held-out 10\% of the training data (patience 3). The second is full fine-tuning of ResNet-18 and MobileNetV2 (AdamW, backbone learning rate $10^{-4}$, head $10^{-3}$, weight decay $10^{-4}$, batch 64, at most 15 epochs with early stopping). The third is ResNet-18 trained from random initialisation on the target data only, which serves as the in-domain control. Because of a 6\,GB GPU limit, full fine-tuning on OCTMNIST used a stratified subset of 20,000 training images and 15 rather than 30 epochs. For a matched-accuracy comparison, we took from each fine-tuning or from-scratch trajectory the epoch whose test accuracy was closest to that of the paired frozen model.}

\rev{Every configuration was run with seeds 42, 43 and 44. The seed controls head training where there is training, the choice of evaluation and calibration images, and the draw of the noise operator.}

\subsection{Perturbation operators and audit}
\label{subsec:audit}

\rev{DCF perturbs each image with six operators. Rotation is by $2^\circ$ (nearest-neighbour, no expansion, uncovered corners set to zero). Brightness and contrast use factors of 1.1 (clipped to $[0,1]$). Gaussian noise has $\sigma=0.01$ on pixels in $[0,1]$ (clipped). Gaussian blur uses a $3\times3$ kernel with $\sigma=0.5$ (reflect padding). JPEG re-encoding uses quality 75. All operators are applied at native resolution, before resizing.}

\rev{An audit characterises how strongly each operator acts. Table~\ref{tab:operators} gives the mean $\ell_2$ distance between clean and perturbed images at the network input, together with the flip rate of the frozen CNNs. The distances differ between operators on the same dataset. On OCTMNIST the $2^\circ$ rotation leaves every image unchanged in the pipeline, so it acts as a null operator there, and only five of the six operators can change a prediction. The spread of the distance across images is small relative to the differences between operators. Excluding the null operator on OCTMNIST, the 95th percentile is at most 1.65 times the median for every operator and dataset. The largest distances occur for rotation on Intel-Scene (median 50.0, 95th percentile 73.3, maximum 89.2). Gaussian noise is almost constant (CIFAR-10: median 2.65, maximum 2.76).}

\rev{A screening check assessed label preservation. A vision-language model judged 420 images (360 perturbed, 20 per operator and dataset for CIFAR-10, STL-10 and Intel-Scene, plus 60 clean controls) without seeing the original label. It assigned the original class to all 420.}

\begin{table}[t]
\centering
\caption{Operator audit. Each cell gives the mean $\ell_2$ distance between clean and perturbed images at the $224\times224$ network input (pixel values in $[0,1]$), followed by the flip rate in percent for frozen CNNs (mean over four models and three seeds).}
\label{tab:operators}
\small
\begin{tabular}{@{}lcccc@{}}
\toprule
Operator & CIFAR-10 & STL-10 & Intel-Scene & OCTMNIST \\
\midrule
Rotation $2^\circ$ & 15.2 / 9.8 & 39.0 / 3.6 & 50.9 / 5.2 & 0.0 / 0.0 \\
Brightness $\times1.1$ & 19.1 / 2.8 & 18.0 / 0.9 & 18.1 / 1.2 & 8.3 / 3.7 \\
Gaussian noise $\sigma=0.01$ & 2.6 / 6.1 & 2.6 / 1.8 & 2.6 / 2.0 & 2.6 / 21.9 \\
Contrast $\times1.1$ & 7.6 / 2.5 & 7.7 / 0.7 & 8.0 / 1.1 & 5.4 / 3.6 \\
Blur $3\times3$, $\sigma=0.5$ & 5.5 / 9.1 & 5.3 / 2.2 & 7.6 / 2.5 & 6.1 / 14.9 \\
JPEG $q=75$ & 9.3 / 21.7 & 8.5 / 5.1 & 7.5 / 3.2 & 0.6 / 2.3 \\
\bottomrule
\end{tabular}
\end{table}

\subsection{Baselines and evaluation measures}
\label{subsec:baselines}

\rev{The consistency scores are compared with single-pass confidence scores and with other perturbation-based scores. For misclassification detection we score every evaluation image with MSP, MSP after temperature scaling, the logit margin of Eq.~(\ref{eq:margin}), negative entropy, maximum logit~\cite{hendrycks2022scaling}, the energy score~\cite{liu2020energy}, DOCTOR~\cite{granese2021doctor}, LS, CS and the CD variants of Section~\ref{subsec:metrics}. The natively trained models receive three further scores. They are the Mahalanobis score on penultimate features~\cite{lee2018simple}, two feature-space consistency scores (the KL divergence between clean and perturbed predictive distributions, and the displacement of the penultimate embedding), and TTA with horizontal flips and two-pixel shifts~\cite{ayhan2018tta}. The TTA operator set is deliberately different from the DCF operators. The zoo adds TTA agreement and TTA mean confidence. The frozen CNNs on CIFAR-10, Intel-Scene and OCTMNIST add a five-head last-layer ensemble~\cite{lakshminarayanan2017simple} with ensemble MSP, entropy, mutual information and member agreement.}

\rev{The main measure is the AUROC for separating correct from incorrect predictions. Higher scores mean ``more likely correct'', so that 0.5 is chance. We also report the area under the risk-coverage curve (AURC; lower is better)~\cite{geifman2017selective}. Under shift we report AUROC on the shifted data, the error of label-free accuracy estimators (average thresholded confidence, ATC~\cite{garg2022atc}, and difference of confidences, DoC~\cite{guillory2021doc}) and, at the model level, rank correlations with accuracy under shift. Cost is measured in forward passes and in wall-clock time with mixed precision on one RTX~4050 laptop GPU.}

\subsection{Shift, adversarial and out-of-distribution tests}
\label{subsec:shiftprotocol}

\rev{The scores are also tested where confidence is expected to be less reliable. We use three kinds of shift. The first is CIFAR-10-C-style corruptions generated in this study (six corruptions, five severities) on the CIFAR-10 frozen CNNs. The second is the official CIFAR-10-C files~\cite{hendrycks2019benchmarking} (19 corruptions, five severities) on the 14 natively trained models. The third is an operator-strength sweep in which every DCF operator is applied at five strengths, level 1 being the setting of Table~\ref{tab:operators}, with the level chosen on the calibration split.}

\rev{For adversarial inputs we ran $\ell_\infty$ PGD~\cite{madry2017towards} with 10 steps, step size $\epsilon_\infty/4$ and a random start, at budgets $\epsilon_\infty\in\{2,4,8\}/255$, on 500 test images per run for the 14 natively trained models. This attack has no knowledge of DCF. An adaptive attack evaluates detection against an attacker who knows the detector. It was run on four of the models (ResNet-20, ResNet-56, VGG-16-BN, RepVGG-A0) with 20 steps of size $\epsilon_\infty/8$ at $\epsilon_\infty\in\{4,8\}/255$. The attack maximises}
\begin{equation}
\mathcal{L}(\mathbf{x}') = \mathrm{CE}\bigl(z(\mathbf{x}'),y\bigr) - \lambda\,\frac{1}{K}\sum_{k=1}^{K}\Bigl[\mathrm{CE}\bigl(z(T_k\mathbf{x}'),\hat{y}(\mathbf{x}')\bigr) + \bigl|c(T_k\mathbf{x}')-c(\mathbf{x}')\bigr|\Bigr],
\label{eq:adaptive}
\end{equation}
\rev{where $\mathrm{CE}$ is the cross-entropy and the weight $\lambda\in\{0,1,5\}$ controls the stability term. The first term pushes the prediction away from the true label, and the second makes the new prediction stable under the DCF operators. Rotation, brightness, contrast and blur are differentiable. JPEG is passed through BPDA with an identity gradient~\cite{athalye2018obfuscated}. The attacker uses its own noise draw, which the defender does not share. For OOD detection we used 1,000 images each from SVHN~\cite{netzer2011svhn} (far OOD) and CIFAR-100 (near OOD).}

\subsection{Statistics, pre-specification and reproducibility}
\label{subsec:stats}

\rev{Every number is a mean over seeds, and models trained once are evaluated on three seeds of sampling. For comparisons between scores we use a hierarchical bootstrap that resamples configurations, then seeds within configurations, then images within runs~\cite{saravanan2020hierarchical} (500 replicates). We report the paired difference with its 95\% interval, Cohen's $d_z$ over configurations, and Wilcoxon signed-rank tests with Holm correction~\cite{holm1979simple} within each family of comparisons. The unit of analysis is the configuration, not the image.}

\rev{Five predictions were specified before the adversarial and OOD experiments were run. They were: (i) consistency scores would detect adversarial inputs and errors under attack better than MSP and margin; (ii) they would not beat energy or Mahalanobis on far OOD; (iii) they would not beat MSP on near OOD; (iv) no consistency-type score would beat MSP for clean error detection; and (v) Eq.~(\ref{eq:cs}) would leave under-trained and stochastic-dropout heads inside or near the normal range, while at least one CD variant would expose them. The grids for $\alpha$ (0 to 1 in steps of 0.25), $\delta$ (0.02, 0.05, 0.10) and $K$ (1 to 6) were fixed before the test results were examined.}

\section{Results}
\label{sec:results}

\subsection{Reference measurements}
\label{subsec:reference}

\rev{Table~\ref{tab:reference} lists the measured DCF quantities for the 28 frozen configurations. Two observations follow. First, the numbers are stable across seeds: the standard deviation of mean CS is at most 0.005. Second, every CAG is negative. At $\delta=0.05$, 11 of the 16 CNN configurations are consistency-dominant and 5 are balanced, while 8 of the 12 transformer configurations are balanced. The transformers are more consistent than the CNNs on CIFAR-10 (CS 0.945 to 0.961 against 0.911 to 0.936). CS values are therefore not comparable across architectures without qualification.}

\begin{table*}[t]
\centering
\caption{Reference values of the DCF quantities (mean over three seeds; $\pm$ is the standard deviation of mean CS across seeds). Acc and ECE are computed on the full test split; LS, CD and CS on $N=1{,}000$ evaluation images. Regime labels use $\delta=0.05$ and are descriptive only.}
\label{tab:reference}
\small
\begin{tabular}{@{}llccccccl@{}}
\toprule
Dataset & Model & Acc & ECE & LS & CD & CS & CAG & Regime \\
\midrule
\multicolumn{9}{@{}l}{\emph{Convolutional networks, frozen backbone and linear head}}\\
CIFAR-10 & ResNet-18 & 0.864 & 0.009 & 0.912 & 0.067 & 0.922\,$\pm$\,0.005 & $-0.058$ & consistency-dom. \\
CIFAR-10 & ResNet-50 & 0.883 & 0.009 & 0.929 & 0.057 & 0.936\,$\pm$\,0.002 & $-0.053$ & consistency-dom. \\
CIFAR-10 & VGG-16 & 0.841 & 0.008 & 0.896 & 0.075 & 0.911\,$\pm$\,0.001 & $-0.070$ & consistency-dom. \\
CIFAR-10 & MobileNetV2 & 0.865 & 0.017 & 0.913 & 0.064 & 0.924\,$\pm$\,0.003 & $-0.059$ & consistency-dom. \\
STL-10 & ResNet-18 & 0.948 & 0.048 & 0.974 & 0.028 & 0.973\,$\pm$\,0.002 & $-0.025$ & balanced \\
STL-10 & ResNet-50 & 0.966 & 0.042 & 0.981 & 0.023 & 0.979\,$\pm$\,0.001 & $-0.013$ & balanced \\
STL-10 & VGG-16 & 0.951 & 0.022 & 0.975 & 0.027 & 0.974\,$\pm$\,0.003 & $-0.023$ & balanced \\
STL-10 & MobileNetV2 & 0.945 & 0.059 & 0.974 & 0.030 & 0.972\,$\pm$\,0.000 & $-0.027$ & balanced \\
Intel-Scene & ResNet-18 & 0.911 & 0.014 & 0.978 & 0.024 & 0.977\,$\pm$\,0.002 & $-0.066$ & consistency-dom. \\
Intel-Scene & ResNet-50 & 0.919 & 0.016 & 0.977 & 0.025 & 0.976\,$\pm$\,0.003 & $-0.057$ & consistency-dom. \\
Intel-Scene & VGG-16 & 0.922 & 0.014 & 0.973 & 0.032 & 0.970\,$\pm$\,0.001 & $-0.048$ & balanced \\
Intel-Scene & MobileNetV2 & 0.917 & 0.019 & 0.971 & 0.026 & 0.972\,$\pm$\,0.001 & $-0.055$ & consistency-dom. \\
OCTMNIST & ResNet-18 & 0.629 & 0.115 & 0.913 & 0.053 & 0.930\,$\pm$\,0.001 & $-0.301$ & consistency-dom. \\
OCTMNIST & ResNet-50 & 0.623 & 0.135 & 0.932 & 0.042 & 0.945\,$\pm$\,0.002 & $-0.322$ & consistency-dom. \\
OCTMNIST & VGG-16 & 0.627 & 0.119 & 0.928 & 0.049 & 0.939\,$\pm$\,0.002 & $-0.312$ & consistency-dom. \\
OCTMNIST & MobileNetV2 & 0.651 & 0.135 & 0.918 & 0.054 & 0.932\,$\pm$\,0.002 & $-0.281$ & consistency-dom. \\
\midrule
\multicolumn{9}{@{}l}{\emph{Vision transformers, frozen backbone and linear head (same recipe)}}\\
CIFAR-10 & ViT-B/16 & 0.951 & 0.011 & 0.959 & 0.036 & 0.961\,$\pm$\,0.003 & $-0.011$ & balanced \\
CIFAR-10 & DeiT-B & 0.943 & 0.013 & 0.957 & 0.042 & 0.958\,$\pm$\,0.001 & $-0.015$ & balanced \\
CIFAR-10 & Swin-B & 0.926 & 0.018 & 0.943 & 0.052 & 0.945\,$\pm$\,0.005 & $-0.019$ & balanced \\
STL-10 & ViT-B/16 & 0.980 & 0.050 & 0.993 & 0.016 & 0.989\,$\pm$\,0.001 & $-0.008$ & balanced \\
STL-10 & DeiT-B & 0.982 & 0.052 & 0.990 & 0.020 & 0.985\,$\pm$\,0.001 & $-0.004$ & balanced \\
STL-10 & Swin-B & 0.980 & 0.081 & 0.991 & 0.026 & 0.982\,$\pm$\,0.001 & $-0.002$ & balanced \\
Intel-Scene & ViT-B/16 & 0.932 & 0.027 & 0.984 & 0.017 & 0.984\,$\pm$\,0.002 & $-0.052$ & consistency-dom. \\
Intel-Scene & DeiT-B & 0.932 & 0.022 & 0.981 & 0.020 & 0.980\,$\pm$\,0.001 & $-0.048$ & balanced \\
Intel-Scene & Swin-B & 0.929 & 0.036 & 0.971 & 0.028 & 0.971\,$\pm$\,0.002 & $-0.042$ & balanced \\
OCTMNIST & ViT-B/16 & 0.656 & 0.118 & 0.935 & 0.039 & 0.948\,$\pm$\,0.000 & $-0.292$ & consistency-dom. \\
OCTMNIST & DeiT-B & 0.611 & 0.161 & 0.935 & 0.044 & 0.946\,$\pm$\,0.001 & $-0.335$ & consistency-dom. \\
OCTMNIST & Swin-B & 0.571 & 0.167 & 0.924 & 0.053 & 0.936\,$\pm$\,0.000 & $-0.364$ & consistency-dom. \\
\bottomrule
\end{tabular}
\end{table*}

\rev{Mean CS is not independent of accuracy. Across the 16 CNN configurations the Pearson correlation between Acc and CS is $r=0.586$ (95\% CI 0.36 to 0.87, $p=0.017$; Spearman 0.71), and CS is not related to ECE ($r=-0.21$, $p=0.43$). Brier scores averaged 0.077 on STL-10, 0.123 on Intel-Scene, 0.198 on CIFAR-10 and 0.485 on OCTMNIST. Like ECE, they follow accuracy. Flip rates depend on the dataset more than on the operator (Table~\ref{tab:operators}). JPEG dominates on CIFAR-10, while noise (21.9\%) and blur (14.9\%) dominate on OCTMNIST. The input distance is a poor guide to the effect on the model. On Intel-Scene the mean distance differs by a factor of about 20 between operators (rotation 50.9, noise 2.6). JPEG moves CIFAR-10 images less than brightness does (9.3 against 19.1) but flips 21.7\% of predictions rather than 2.8\%.}

\subsection{Error detection against confidence baselines}
\label{subsec:e3}

\rev{The central comparison is with single-pass confidence scores. Table~\ref{tab:e3} gives error-detection AUROC and AURC for the frozen CNNs. MSP, its temperature-scaled version and the margin are the best scores on every dataset except OCTMNIST, where negative entropy is marginally ahead. CS is below MSP on all four datasets and in all 16 configurations. LS alone is much weaker (0.700). The hierarchical bootstrap puts the difference CS $-$ MSP at $-0.057$ (95\% CI $-0.073$ to $-0.043$; Holm-corrected $p=1.2\times10^{-4}$; $d_z=-1.90$). The difference CS $-$ margin, computed on unrounded values, is $-0.056$ ($-0.070$ to $-0.041$), and CS $-$ negative entropy is $-0.051$ ($-0.069$ to $-0.036$). CS does beat the energy score on average ($+0.034$), but the interval includes zero ($-0.009$ to $+0.082$; $p=0.18$). Energy is itself a weak misclassification detector here because the linear heads leave the logit scale uncalibrated.}

\begin{table*}[t]
\centering
\caption{Misclassification detection for the 16 frozen CNN configurations (48 runs). Each cell gives AUROC / AURC, averaged over models and seeds. Higher AUROC and lower AURC are better.}
\label{tab:e3}
\small
\begin{tabular}{@{}lccccc@{}}
\toprule
Score & CIFAR-10 & STL-10 & Intel-Scene & OCTMNIST & All \\
\midrule
MSP & \textbf{0.883} / 0.029 & 0.950 / 0.004 & \textbf{0.892} / 0.014 & 0.778 / 0.174 & \textbf{0.876} / 0.055 \\
MSP, temperature-scaled & \textbf{0.883} / 0.029 & \textbf{0.951} / 0.004 & \textbf{0.892} / 0.014 & 0.778 / 0.174 & \textbf{0.876} / 0.055 \\
Logit margin & \textbf{0.883} / 0.029 & 0.950 / 0.004 & 0.891 / 0.014 & 0.773 / 0.175 & 0.874 / 0.056 \\
Negative entropy & 0.873 / 0.031 & 0.940 / 0.004 & 0.884 / 0.015 & \textbf{0.782} / 0.173 & 0.870 / 0.056 \\
Max logit & 0.813 / 0.044 & 0.910 / 0.006 & 0.722 / 0.040 & 0.773 / 0.182 & 0.804 / 0.068 \\
Energy & 0.794 / 0.048 & 0.890 / 0.008 & 0.686 / 0.045 & 0.767 / 0.185 & 0.784 / 0.071 \\
CS, Eq.~(\ref{eq:cs}) & 0.802 / 0.045 & 0.918 / 0.006 & 0.861 / 0.018 & 0.693 / 0.215 & 0.819 / 0.071 \\
CS, rescaled CD & 0.820 / 0.041 & 0.928 / 0.005 & 0.870 / 0.017 & 0.713 / 0.204 & 0.833 / 0.067 \\
LS, Eq.~(\ref{eq:ls}) & 0.737 / 0.076 & 0.766 / 0.021 & 0.683 / 0.059 & 0.616 / 0.304 & 0.700 / 0.115 \\
\bottomrule
\end{tabular}
\end{table*}

\rev{The same ordering holds beyond the frozen-backbone setting (Fig.~\ref{fig:e3}, Table~\ref{tab:families}). With the transformers, CS reached 0.855 against 0.891 for MSP and 0.895 for the margin. It was ahead of MSP only in 6 of 36 runs, all on STL-10, where it still trailed the margin. With the natively trained CIFAR-10 models the gap narrowed to 0.016 but did not change sign: CS was better than MSP in 3 of 42 runs. The deficit is therefore smaller for natively trained models than for linear probes, but it has the same sign. In the 13-model zoo, CS (0.778) was again below MSP (0.865). TTA mean confidence matched MSP exactly (0.865), while TTA label agreement (0.742) was below CS. The five-head ensembles were barely better than a single head (ensemble MSP 0.854 against 0.851), but CS lost to ensemble MSP in all 36 runs.}

\rev{Among the consistency-type scores computed on the natively trained models, TTA mean confidence was the only one to beat MSP (0.939 against 0.920). It averages predictions over views instead of measuring their disagreement. This is consistent with Proposition~\ref{prop:redundancy}: averaging can improve the estimate of the clean posterior, while agreement only adds a second, noisier view of the same quantity.}

\begin{figure*}[t]
\centering
\includegraphics[width=0.92\textwidth]{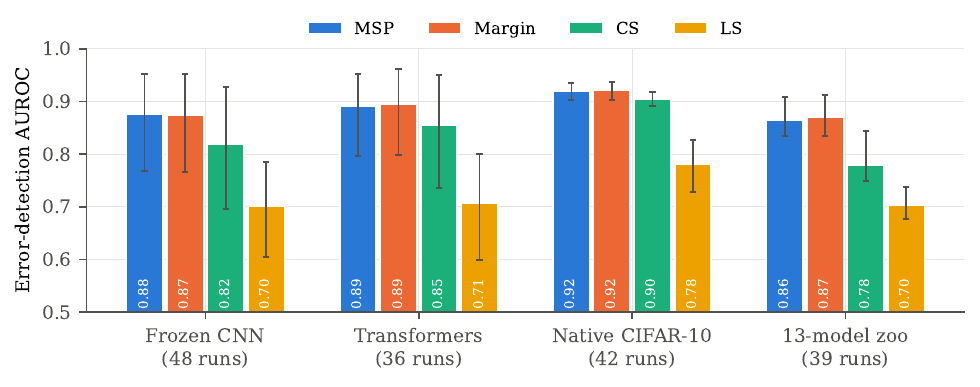}
\caption{Misclassification-detection AUROC on clean data for four model families. Bars are means over runs; whiskers span the 10th to 90th percentiles across runs. CS (Eq.~(\ref{eq:cs})) is below MSP and the margin in every family.}
\label{fig:e3}
\end{figure*}

\begin{table}[t]
\centering
\caption{Clean misclassification-detection AUROC beyond the frozen CNNs (means over runs). ``CS $>$ MSP'' counts the runs in which CS had the higher AUROC.}
\label{tab:families}
\small
\setlength{\tabcolsep}{4pt}
\begin{tabular}{@{}lcccccc@{}}
\toprule
Family (runs) & MSP & Margin & Best other & CS & LS & CS $>$ MSP \\
\midrule
Transformers (36) & 0.891 & 0.895 & entropy 0.878 & 0.855 & 0.707 & 6/36 (all STL-10) \\
Native CIFAR-10 (42) & 0.920 & 0.921 & TTA mean conf.\ 0.939 & 0.904 & 0.781 & 3/42 \\
13-model zoo (39) & 0.865 & 0.870 & TTA mean conf.\ 0.865 & 0.778 & 0.704 & 0/39 \\
Last-layer ensembles (36) & 0.851 & n/a & ensemble MSP 0.854 & 0.786 & 0.678 & 0/36 vs ens.\ MSP \\
\bottomrule
\end{tabular}
\end{table}

\rev{Three further analyses examine whether LS or CD contribute beyond confidence. Among predictions with MSP $\ge 0.90$ (about 600 images per run with 15 errors), LS was at chance (AUROC 0.508), CS reached 0.692 and the margin 0.773. A cross-validated logistic model on MSP and margin changed by $+0.001$ when LS was added and by $-0.001$ when LS and CD were both added. The Spearman correlation between margin and LS was 0.40 to 0.68 across configurations (mean 0.55), but the partial correlation given MSP fell to $-0.03$ to 0.16 (mean 0.07). Stability is mostly a re-expression of confidence.}

\rev{The cost of DCF is also higher than that of the baselines. DCF needs seven forward passes per image, and the measured slowdown with mixed precision was $7.9\times$ (VGG-16) to $15.2\times$ (MobileNetV2). The ranking of configurations by mean CS also depends on the operator set. Kendall's $\tau_{\mathrm{K}}$ between the ranking from all six operators and the ranking from a subset was 0.44 on average for one operator, 0.68 for three and 0.88 for five. Across the five severities of the self-generated CIFAR-10 corruptions (Section~\ref{subsec:shift}), the ranking of the four frozen CIFAR-10 CNNs by mean CS did not change ($\tau_{\mathrm{K}}=1$ for every pair of severities). With only four models this is a weak test, and it concerns corruption severity, not operator strength.}

\subsection{Reformulating CD and choosing $\alpha$}
\label{subsec:cdvariants}

\rev{Eq.~(\ref{eq:cd}) can reward a model that switches classes at constant confidence. Three alternative definitions are evaluated. The first is the change in the probability of the clean class, $|p_{\hat{y}_i}(\mathbf{x}_i)-p_{\hat{y}_i}(\mathbf{x}_i^{(k)})|$. The second is the total-variation distance between the clean and perturbed probability vectors. The third is Eq.~(\ref{eq:cd}) rescaled to $[0,1]$ on the calibration split. The first two respond to a change of class. The rescaled version keeps the maximum-probability form and therefore keeps this property. Each is evaluated with $\alpha$ selected on the calibration split (Table~\ref{tab:cdvariants}, upper part). The total-variation and rescaled versions improve CS slightly (0.828 and 0.833 against 0.819), the clean-class version does not, and none comes within 0.04 of MSP. Choosing $\alpha$ on the calibration split gave the same test AUROC as $\alpha=0.5$ to three decimals. Even the best $\alpha$ chosen on the test data improved AUROC by less than 0.001. The value $\alpha=0.5$ is therefore used throughout, and CS is treated as a ranking score and not as an improvement over its components. Logit displacement is not used as a CD variant, because it is unbounded and depends on the logit scale. Rescaling changes the range of CD, not its calibration, and no temperature-scaled CD is evaluated.}

\rev{The lower part of Table~\ref{tab:cdvariants} reports heads retrained on ResNet-20 features. These heads test whether the variants expose abnormal models. With dropout rate 0.9 kept active at test time, LS is 0.249, and Eq.~(\ref{eq:cd}) lifts CS to 0.477. The total-variation variant gives 0.346. All variants separate stochastic-dropout heads from normal ones, since the 14 natively trained models have mean CS between 0.956 and 0.977. The opposite case is harder. A head trained for only 50 SGD steps has accuracy 0.793 but a mean confidence of 0.255, close to uniform. Its confidence barely moves under perturbation, so CD is near zero and CS stays at 0.937 to 0.945 under every variant. This is within 0.02 of the fully trained head (0.942 to 0.955) and inside the range of the 28 reference configurations in Table~\ref{tab:reference}. The model is under-confident, not unstable. No consistency score detects it, and only the gap between accuracy and confidence does. The same occurred in the zoo, where an under-trained ShuffleNetV2 head with ECE 0.59 had CS 0.949. Prediction (v) was therefore partly confirmed. The variants expose stochastic dropout, but nothing in DCF exposes under-confidence.}

\begin{table}[t]
\centering
\caption{Upper part: error-detection AUROC of CS with four definitions of CD, with $\alpha=0.5$ and with $\alpha$ chosen on the calibration split (48 frozen-CNN runs; MSP reaches 0.876). Lower part: ResNet-20 features with retrained heads on CIFAR-10 (mean of three seeds).}
\label{tab:cdvariants}
\small
\setlength{\tabcolsep}{4pt}
\begin{tabular}{@{}lccccc@{}}
\toprule
CD definition & \multicolumn{2}{c}{$\alpha=0.5$} & \multicolumn{2}{c}{$\alpha$ from calibration} & $\alpha=0.5$ chosen \\
\midrule
Eq.~(\ref{eq:cd}), max-probability & \multicolumn{2}{c}{0.819} & \multicolumn{2}{c}{0.819} & 58\% of runs \\
Rescaled to $[0,1]$ on calibration & \multicolumn{2}{c}{0.833} & \multicolumn{2}{c}{0.833} & 21\% \\
Clean-class probability & \multicolumn{2}{c}{0.817} & \multicolumn{2}{c}{0.817} & 44\% \\
Total variation & \multicolumn{2}{c}{0.828} & \multicolumn{2}{c}{0.828} & 46\% \\
\midrule
Head & Acc & Mean conf. & LS & CS, Eq.~(\ref{eq:cs}) & CS, total variation \\
\midrule
Fully trained & 0.921 & 0.963 & 0.943 & 0.955 & 0.942 \\
Stochastic dropout rate 0.5 & 0.867 & 0.892 & 0.819 & 0.843 & 0.801 \\
Stochastic dropout rate 0.9 & 0.429 & 0.508 & 0.249 & 0.477 & 0.346 \\
Under-trained, 50 SGD steps & 0.793 & 0.255 & 0.909 & 0.945 & 0.937 \\
\bottomrule
\end{tabular}
\end{table}

\subsection{Why consistency is redundant on clean data}
\label{subsec:mechanism}

\rev{The clean-data results raise the question of why consistency adds so little. Proposition~\ref{prop:redundancy} predicts that LS can only beat the margin if the logit displacement separates correct from incorrect images at a fixed margin. We therefore binned the evaluation images of each of the 84 frozen runs (48 CNN, 36 transformer) into margin deciles. Within each bin with at least five correct and five incorrect images, we computed the AUROC of the negative mean displacement. The size-weighted average was 0.515. It exceeded 0.5 in 37 of 84 runs, and the run-level Wilcoxon test against 0.5 gave $p=0.67$. LS reached 0.535 within bins. Adding the displacement and LS to a cross-validated logistic model on the margin changed AUROC by $+0.003$ on average. Over the whole evaluation set the displacement is almost uncorrelated with the margin (Spearman $-0.02$ for CNNs). The images that flip are therefore mostly those with small margins, and the margin already ranks them. This accounts for the negative results of Section~\ref{subsec:e3}. On clean data the operators move correct and incorrect images by similar amounts, and whether a move causes a flip is decided by a quantity the clean output already contains.}

\subsection{The certificate and the bounds}
\label{subsec:certres}

\rev{The certificate rests on the flip implication, which can be checked directly. The margin is a $\sqrt{2}$-Lipschitz function of the logits, and the measured ratio $|\gamma(\mathbf{x}_i^{(k)})-\gamma(\mathbf{x}_i)|/\|\Delta_{i,k}\|_2$ never exceeded $\sqrt{2}$ (largest value 1.414). No image with $\gamma(\mathbf{x}_i)>\sqrt{2}\|\Delta_{i,k}\|_2$ flipped. This confirms only that the logits were recorded and processed correctly (Remark~\ref{rem:withdraw}).}

\rev{Table~\ref{tab:cert} and Fig.~\ref{fig:cert} report the conformal certificate of Proposition~\ref{prop:conformal} on all 84 frozen runs. The guarantee held in every run at every $\beta$. The largest realised value of $P(\text{certified}\wedge\text{flip})$ was 0.003 at $\beta=0.05$. Certified images were also far more often correct (0.989 against 0.826 at $\beta=0.05$), so the certificate doubles as a cheap high-precision filter. Its coverage depends strongly on the setting. At $\beta=0.05$ it certified 1.2\% of CIFAR-10 images for CNNs but 63.5\% of STL-10 images for transformers, where the margins of the models are large relative to their displacements. As a check of exchangeability, we also tested $P(R(\mathbf{x})>\hat{q})\le\beta$ directly. It exceeded $\beta$ by more than 1.645 standard errors in 6 of 84 runs at each $\beta$, close to the 4.2 expected by chance at that level. Under shift the certificate certifies few images. On CIFAR-10 corruptions the calibration quantile computed on clean data certified only 0.02\% to 0.8\% of shifted images (against 0.1\% to 3.3\% of clean images in the same runs), and none of them flipped.}

\begin{table}[t]
\centering
\caption{Conformal no-flip certificate (Proposition~\ref{prop:conformal}) on 84 frozen runs. The guarantee is $P(\text{certified}\wedge\text{flip})\le\beta$.}
\label{tab:cert}
\small
\begin{tabular}{@{}cccccc@{}}
\toprule
$\beta$ & Share certified & $P(\text{cert.}\wedge\text{flip})$ & Worst run & Runs violating & Acc.\ certified / not \\
\midrule
0.01 & 12.6\% & 0.0000 & 0.001 & 0/84 & 0.977 / 0.841 \\
0.05 & 23.2\% & 0.0002 & 0.003 & 0/84 & 0.989 / 0.826 \\
0.10 & 29.5\% & 0.0005 & 0.005 & 0/84 & 0.992 / 0.814 \\
\bottomrule
\end{tabular}
\end{table}

\begin{figure*}[t]
\centering
\includegraphics[width=0.92\textwidth]{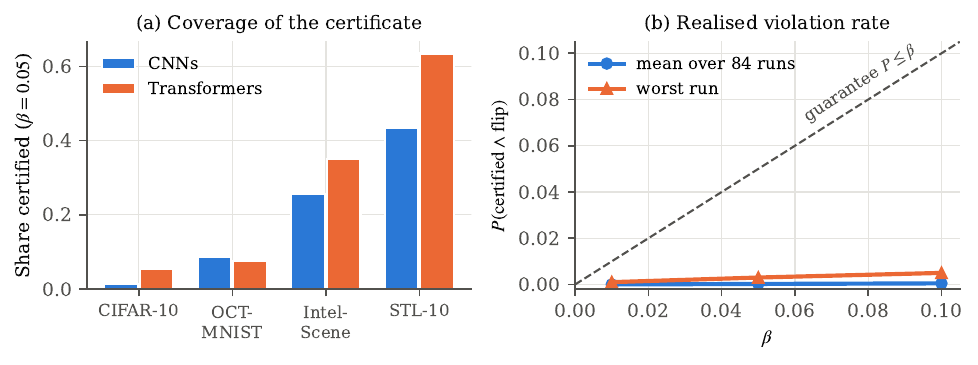}
\caption{Conformal no-flip certificate. (a) Share of evaluation images certified at $\beta=0.05$, by dataset and model family. (b) Realised $P(\text{certified}\wedge\text{flip})$ against the guarantee $\beta$ (dashed), as the mean over 84 runs and for the worst run.}
\label{fig:cert}
\end{figure*}

\rev{The certificate is a guarantee about flips, so a natural comparison is with per-sample flip bounds. We compared these bounds on tightness, calibration and validity. For each operator we bounded the flip probability of an image from its margin, using quantities measured on the calibration split only. We compared each bound with the observed flip rate in margin deciles of the evaluation split (2,880 decile cells on the 48 CNN runs; mean observed flip rate 0.054). Brier scores are computed per decile cell unless marked per image. A bound shaped like a Lipschitz flip lemma equals one when the margin is below the largest calibration displacement and zero otherwise. It averaged 0.599, was trivial (equal to one) in 55\% of cells and had a Brier score of 0.494. First- and second-moment Markov bounds averaged 0.484 and 0.388 (Brier 0.273 and 0.225). The empirical tail $P_{\mathrm{cal}}(\sqrt{2}\|\Delta_k\|_2\ge\gamma)$ averaged 0.229 with Brier 0.098 but was exceeded in 0.07\% of cells, because it is not a finite-sample guarantee. The DKW correction (95\% confidence) gives a finite-sample guarantee. The corrected bound averaged 0.268 (Brier 0.148 per image) in this comparison. Evaluated on all 84 frozen runs, CNNs and transformers, it averaged 0.235 (Brier 0.124 per image) and was exceeded in none of 5,040 cells. The calibrated bounds are informative in a way the lemma-shaped bound is not, although they remain upper bounds about four to five times the observed flip rate.}

\subsection{Distribution shift: where consistency helps and where it does not}
\label{subsec:shift}

\rev{The first question is whether the clean-data ordering survives a shift. Under the self-generated CIFAR-10 corruptions, CS detected errors less well than MSP (0.737 against 0.815; margin 0.807, energy 0.767). Adding LS and CD to a confidence-and-margin model left the AUROC unchanged at 0.814. On the official CIFAR-10-C with the 14 natively trained models, pooled over all 19 corruptions and five severities, CS (0.828) was again slightly below MSP (0.835), DOCTOR (0.836) and negative entropy (0.837).}

\rev{Pooled results hide differences between corruption families (Fig.~\ref{fig:c10c}a). On the four noise corruptions (Gaussian, shot, speckle, impulse), CS was \emph{better} than the margin by 0.05 to 0.11 AUROC. On blur it was worse by about 0.04, and on weather and digital corruptions it was close to or below the margin. The difference follows how badly MSP is misled. Across the 19 corruptions, the gain from adding DCF features to confidence features grew with MSP over-confidence (mean confidence minus accuracy) with Spearman $\rho=0.88$ ($p=6.5\times10^{-7}$; Fig.~\ref{fig:c10c}b). Noise is where confidence is most inflated (over-confidence 0.38 on average) and where the gain is largest (0.13 when the combiner is fitted on the same corruption). The most useful single operator under noise was blur. The score computed from blur alone reached 0.878 against 0.753 for the margin. Blur partly removes the noise, so a prediction that changes after blurring is likely to have been driven by the noise. This is the mechanism of feature squeezing~\cite{xu2018feature}. It is also the situation in which the condition of Proposition~\ref{prop:redundancy} fails: at the same corrupted clean logits, correct and incorrect predictions respond differently to a denoising operator.}

\begin{figure*}[t]
\centering
\includegraphics[width=0.92\textwidth]{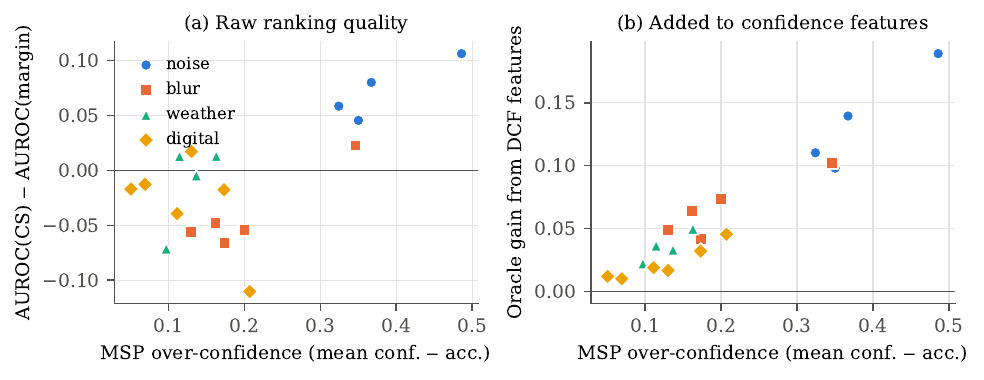}
\caption{Official CIFAR-10-C, 14 natively trained models, one point per corruption (mean over models, seeds and severities). (a) AUROC of CS minus AUROC of the margin. (b) Gain in cross-validated AUROC from adding DCF features to confidence features when the combiner is fitted on labelled data from the same corruption. Both grow with MSP over-confidence; noise corruptions are the only family where CS beats the margin outright.}
\label{fig:c10c}
\end{figure*}

\rev{The gain in Fig.~\ref{fig:c10c}b uses labelled examples of the corruption being tested, and deployed systems typically lack such examples. When the combiner was trained on clean data only, the gain on noise corruptions fell to $+0.009$. When it was trained on all other corruption families and tested on noise, the gain was $+0.002$ (not significant). For blur it became $-0.046$ and for digital corruptions $-0.012$. Only when other noise corruptions were in the training set (leave-one-corruption-out) did a gain survive on noise ($+0.067$). Consistency features therefore help only under a shift that has been observed and labelled.}

\rev{Changing the operators does not remove this limitation. Increasing the strength of all six operators made CS worse, not better. Its AUROC fell from 0.903 at level 1 to 0.755 at level 5, while MSP stayed at 0.920. When the level was chosen on the calibration split, level 1 (the setting of Table~\ref{tab:operators}) was selected in 40 of 42 runs.}

\rev{Finally, we asked whether clean-data scores predict accuracy under shift. For label-free accuracy estimation under the self-generated shift, average thresholded confidence with MSP had a mean absolute error of 0.049, against 0.124 when CS replaced MSP and 0.091 for average confidence. At the model level, in the 13-model zoo, clean accuracy predicted accuracy on held-out corruption families almost perfectly (Spearman 0.99), and clean CS added nothing once accuracy was controlled for (partial Spearman $+0.08$). Clean CS did correlate with effective robustness, the residual from the accuracy-on-the-line fit~\cite{miller2021accuracy} (Spearman 0.79). The correlation was driven by one anomalous model, the ShuffleNetV2 head with ECE 0.59. With that model removed and the line refitted on a logit scale, it fell to 0.19 ($p=0.56$).}

\subsection{Thresholds and regimes}
\label{subsec:thresholds}

\rev{The same question applies to the regime labels of CAG. The labels of Eq.~(\ref{eq:cag}) depend on $\delta$. Moving $\delta$ from 0.02 to 0.10 changes the label of 11 of the 16 CNN configurations. At $\delta=0.02$, 15 are consistency-dominant, and at $\delta=0.10$ only the four OCTMNIST models are. No configuration is accuracy-dominant at any $\delta$ in the grid. The clean-data quantities are also evaluated as predictors of behaviour under shift. A shift alarm with a threshold chosen on some corruptions and tested on others reached a balanced accuracy of 0.72 with CS, 0.76 with LS, 0.91 with the margin and 0.94 with MSP. Across the CIFAR-10 runs, clean CAG did not predict the drop in accuracy under shift (Spearman 0.09, $p=0.78$). A consistency floor $\tau$ on mean CS, used to separate unstable from stable models, would not bind. Every configuration in Table~\ref{tab:reference} has mean CS above 0.91, and the only heads below any reasonable floor are the artificial stochastic-dropout heads of Table~\ref{tab:cdvariants}. Neither $\delta$ nor $\tau$ is validated on independent data, so $\delta$ is a descriptive tolerance without an operational interpretation.}

\subsection{Adversarial and out-of-distribution inputs}
\label{subsec:adv}

\rev{Adversarial and OOD inputs test the settings in which confidence is expected to fail. Table~\ref{tab:adv} summarises the pre-specified tests on the 14 natively trained models. Prediction (i) was confirmed. Against non-adaptive PGD, CS separated clean from attacked images with AUROC 0.83, 0.80 and 0.73 at $\epsilon_\infty=2$, 4 and 8/255, while MSP fell to 0.370, 0.105 and 0.023, because the attack makes the wrong class very confident. CS beat MSP and the margin in 42 of 42 runs at every budget. It also ranked errors on the mixed clean-plus-attacked set far better (0.87 against 0.32 at $\epsilon_\infty=2/255$). CS was not the best detector, however. The displacement of the penultimate embedding under the same operators reached 0.89 to 0.94, and the Mahalanobis score overtook CS at $\epsilon_\infty=8/255$ (0.88 against 0.73). Predictions (ii) and (iii) were also confirmed. On SVHN, CS (0.841) was below energy and Mahalanobis (both 0.889), and on CIFAR-100 it was below MSP (0.844 against 0.853). One result was not predicted: a CS computed from the noise operator alone detected SVHN better than MSP (0.910 against 0.890, 39 of 42 runs). Prediction (iv) failed in one respect, since TTA mean confidence beat MSP on clean data (0.939 against 0.920). This score averages views rather than measuring their disagreement (Section~\ref{subsec:e3}).}

\begin{table}[t]
\centering
\caption{Pre-specified adversarial and OOD tests on 14 natively trained CIFAR-10 models (42 runs). AUROC for separating clean from attacked or OOD images (higher is better). The attack is non-adaptive $\ell_\infty$ PGD.}
\label{tab:adv}
\small
\setlength{\tabcolsep}{4pt}
\begin{tabular}{@{}lccccc@{}}
\toprule
Score & PGD $2/255$ & PGD $4/255$ & PGD $8/255$ & SVHN & CIFAR-100 \\
\midrule
MSP & 0.370 & 0.105 & 0.023 & 0.890 & 0.853 \\
Energy & 0.366 & 0.106 & 0.022 & 0.889 & 0.853 \\
Mahalanobis & 0.723 & 0.800 & \textbf{0.875} & 0.889 & 0.801 \\
CS, all six operators & 0.829 & 0.800 & 0.726 & 0.841 & 0.844 \\
CS, noise operator only & 0.583 & 0.224 & 0.061 & \textbf{0.910} & 0.846 \\
LS & 0.761 & 0.774 & 0.709 & 0.684 & 0.725 \\
Embedding drift & \textbf{0.914} & \textbf{0.936} & 0.888 & 0.768 & 0.811 \\
\bottomrule
\end{tabular}
\end{table}

\rev{Detection against an attack that ignores the detector is weak evidence~\cite{tramer2020adaptive}. Under the adaptive attack of Eq.~(\ref{eq:adaptive}) (Table~\ref{tab:adaptive}, Fig.~\ref{fig:adv}b), the attacker can evade DCF. With $\lambda=5$ the detection AUROC of CS drops to 0.09 at $\epsilon_\infty=4/255$ and 0.05 at $\epsilon_\infty=8/255$, below chance, because the attacked images become more stable than clean ones. Evasion has a cost. With $\lambda=5$ only 6\% (at $4/255$) and 9\% (at $8/255$) of attacks still change the prediction, against 98\% and 100\% for plain PGD. At every $\lambda$, CS ranked the resulting errors well (AUROC 0.79 to 0.87), while MSP did not (0.06 to 0.50). Under this one objective, the attacker can either fool the model or stay stable under the operators, but rarely both. Stronger adaptive objectives were not tested, and no claim of adversarial robustness is made.}

\begin{table}[t]
\centering
\caption{Adaptive PGD against DCF (Eq.~(\ref{eq:adaptive})) on four natively trained models, three seeds each. Detection: clean versus attacked. Error detection: correct versus incorrect predictions on attacked images.}
\label{tab:adaptive}
\small
\begin{tabular}{@{}cccccc@{}}
\toprule
$\epsilon_\infty$ & $\lambda$ & Attack success & CS detection & CS error det. & MSP error det. \\
\midrule
4/255 & 0 & 0.98 & 0.81 & 0.85 & 0.10 \\
4/255 & 1 & 0.70 & 0.67 & 0.87 & 0.23 \\
4/255 & 5 & 0.06 & 0.09 & 0.86 & 0.50 \\
8/255 & 0 & 1.00 & 0.76 & 0.80 & 0.06 \\
8/255 & 1 & 0.86 & 0.67 & 0.79 & 0.11 \\
8/255 & 5 & 0.09 & 0.05 & 0.82 & 0.43 \\
\bottomrule
\end{tabular}
\end{table}

\begin{figure*}[t]
\centering
\includegraphics[width=0.92\textwidth]{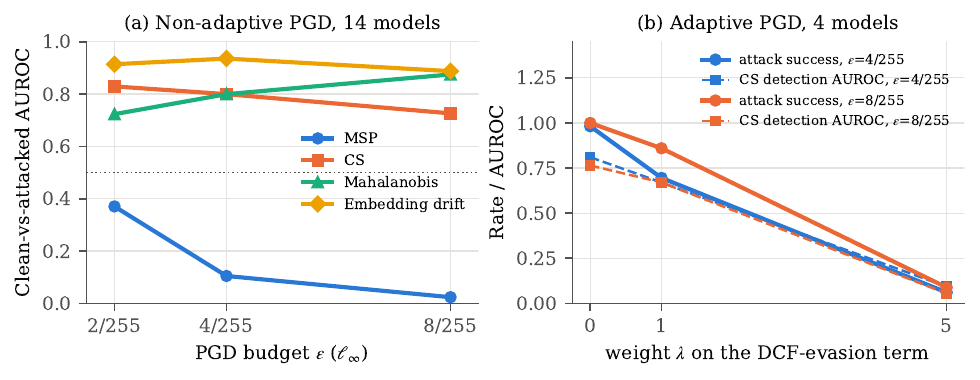}
\caption{Adversarial inputs. (a) Clean-versus-attacked AUROC under non-adaptive PGD; MSP falls below chance because the attack raises confidence. (b) Adaptive PGD (Eq.~(\ref{eq:adaptive})): as the weight $\lambda$ on stability grows, CS detection collapses, but so does the attack success rate.}
\label{fig:adv}
\end{figure*}

\subsection{Stable but wrong: OCTMNIST}
\label{subsec:octres}

\rev{The negative CAG on OCTMNIST is far larger than on the other datasets (Table~\ref{tab:reference}). One reading is backbone-domain mismatch: an ImageNet backbone with a short linear probe may produce stable features that do not separate retinal classes. Adaptation controls evaluate this reading directly (Table~\ref{tab:e6}).}

\begin{table}[t]
\centering
\caption{Adaptation controls (mean of three seeds). Frozen: ImageNet backbone with a linear head. ES: early stopping. The scratch model sees no ImageNet data. Validation-split CAG uses the official OCTMNIST validation split, which has the training prior.}
\label{tab:e6}
\small
\setlength{\tabcolsep}{4pt}
\begin{tabular}{@{}llcccccc@{}}
\toprule
Dataset & Model, adaptation & Acc & LS & CS & CAG (test) & CAG (val) \\
\midrule
OCTMNIST & ResNet-18, frozen, 5 epochs & 0.629 & 0.913 & 0.930 & $-0.301$ & $-0.128$ \\
 & ResNet-18, frozen, $\le$30 epochs + ES & 0.622 & 0.917 & 0.931 & $-0.310$ & $-0.124$ \\
 & ResNet-18, full fine-tuning & 0.742 & 0.877 & 0.898 & $-0.156$ & $-0.021$ \\
 & ResNet-18, from scratch (in-domain) & 0.753 & 0.953 & 0.959 & $-0.206$ & $-0.078$ \\
 & MobileNetV2, frozen, 5 epochs & 0.651 & 0.918 & 0.932 & $-0.281$ & $-0.098$ \\
 & MobileNetV2, full fine-tuning & 0.739 & 0.813 & 0.862 & $-0.123$ & $+0.002$ \\
\midrule
Intel-Scene & ResNet-18, frozen, 5 epochs & 0.911 & 0.978 & 0.977 & $-0.066$ & n/a \\
 & ResNet-18, full fine-tuning & 0.924 & 0.981 & 0.981 & $-0.058$ & n/a \\
 & ResNet-18, from scratch (in-domain) & 0.865 & 0.975 & 0.977 & $-0.112$ & n/a \\
\bottomrule
\end{tabular}
\end{table}

\rev{If weak adaptation caused the gap, stronger adaptation should remove it. It does not: CAG stayed below $-0.10$ in all 33 OCTMNIST runs. These are 11 configurations with three seeds each. Eight are frozen CNNs: each of the four CNNs of Table~\ref{tab:reference} with a head trained for 5 epochs, and with a head trained for up to 30 epochs with early stopping. The other three are ResNet-18 and MobileNetV2 with full fine-tuning, and ResNet-18 trained from scratch. The transformers are not among these runs. Table~\ref{tab:e6} lists the configurations involving ResNet-18 and MobileNetV2. Longer head training changed nothing. Full fine-tuning raised accuracy by about 0.1 and improved CAG by $0.145\pm0.077$ (ResNet-18) and $0.158\pm0.012$ (MobileNetV2), but the gap stayed at $-0.12$ to $-0.16$. On Intel-Scene the gap stayed between $-0.05$ and $-0.07$ for frozen and fine-tuned models and was larger ($-0.11$) for the less accurate model trained from scratch.}

\rev{Backbone-domain mismatch is also tested directly. A ResNet-18 trained only on OCTMNIST, which has no ImageNet features to mismatch, still had CAG $-0.206$. At matched accuracy the comparison is sharper. At the epoch where the in-domain model reached the accuracy of the frozen ImageNet model (0.620), its CAG was $-0.331$, against $-0.308$ for the frozen model. On Intel-Scene, a fine-tuned model at matched accuracy had the same gap as the frozen one ($-0.068$ against $-0.067$). These controls do not support backbone-domain mismatch as the sole explanation. They do not exclude a contribution from it, because each matched-accuracy comparison uses one point per trajectory.}

\rev{A second explanation is label shift, which accounts for part of the gap. The OCTMNIST training and validation splits have the class prior CNV 0.344, DME 0.105, drusen 0.080 and normal 0.472, while the official test split is balanced. Models trained on the skewed prior rarely predict the minority classes. Drusen recall was 12\% for the frozen ResNet-18 and 33\% after full fine-tuning, and the wrong predictions for drusen images were as stable as any others. On the validation split, which shares the training prior, the gap nearly vanishes for fine-tuned models ($-0.021$ and $+0.002$) and is much smaller for frozen ones (about $-0.10$ to $-0.13$; Table~\ref{tab:e6} and Fig.~\ref{fig:labelshift}a). Two experiments test this causally.}

\rev{First, we corrected the saved test logits for the prior shift by adding the difference $\log\pi_{\mathrm{test}}-\log\pi_{\mathrm{train}}$ between the log class priors of the test data ($\pi_{\mathrm{test}}$) and of the training data ($\pi_{\mathrm{train}}$)~\cite{saerens2002adjusting}. With the true (balanced) test prior, mean CAG over all OCTMNIST models and modes improved from $-0.264$ to $-0.162$ (32 of 33 runs improved), and drusen recall rose from 14\% to 53\%. With a prior estimated without labels by expectation-maximisation (EM), CAG improved to $-0.224$ (29 of 33 runs). Prior correction thus recovers about 40\% of the gap on OCTMNIST. The rest remains on the test split, and a residual of about $-0.10$ to $-0.13$ is present on the validation split for frozen models. There, CS saturates above accuracy.}

\rev{Second, we induced the same kind of shift on CIFAR-10, where the frozen ResNet-18 has a small gap. Heads were trained with two classes (bird and cat) cut to 10\% or 2\% of their images and tested on the balanced test set (Fig.~\ref{fig:labelshift}b and Table~\ref{tab:labelshift}). At 2\%, the rare classes were recognised 9\% of the time, yet their LS stayed at 0.82, and CAG fell from $-0.060$ to $-0.175$. Correcting the prior brought it back to $-0.083$, and re-weighting the test set to the training prior brought it to $-0.031$. Stable but wrong behaviour can thus be induced by label shift and largely removed by correcting it.}

\begin{table}[t]
\centering
\caption{Induced label shift on CIFAR-10 (ResNet-18, frozen, three seeds). Two classes are kept at the given fraction of their training images; the test set is balanced. The 100\% row is a separate run and differs slightly from the value in Table~\ref{tab:reference}.}
\label{tab:labelshift}
\small
\begin{tabular}{@{}cccccccc@{}}
\toprule
Kept & Acc & Rare recall & Rare LS & CAG & CAG, prior corrected & CAG, re-weighted \\
\midrule
100\% & 0.863 & 0.77 & 0.87 & $-0.060$ & $-0.060$ & $-0.060$ \\
10\% & 0.821 & 0.48 & 0.81 & $-0.101$ & $-0.070$ & $-0.040$ \\
2\% & 0.749 & 0.09 & 0.82 & $-0.175$ & $-0.083$ & $-0.031$ \\
\bottomrule
\end{tabular}
\end{table}

\begin{figure*}[t]
\centering
\includegraphics[width=0.92\textwidth]{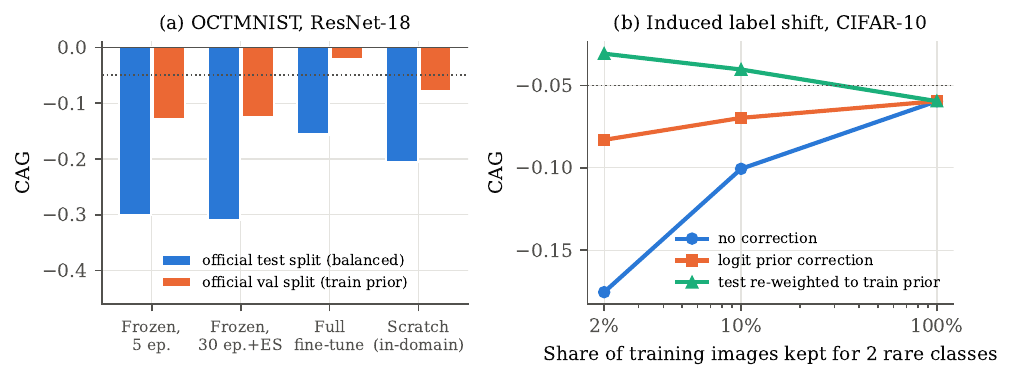}
\caption{Stable but wrong. (a) CAG of ResNet-18 on OCTMNIST under four adaptation modes, on the balanced test split and on the validation split that shares the training prior. (b) Induced label shift on CIFAR-10: CAG worsens as two classes become rare, and prior correction or re-weighting removes most of the effect. The dotted line marks CAG $=-0.05$, the boundary for $\delta=0.05$.}
\label{fig:labelshift}
\end{figure*}

\section{Discussion}
\label{sec:discussion}

\subsection{What the results say about perturbation consistency}

\rev{On clean data and under shifts that the system has not seen, perturbation consistency was a costlier and noisier re-expression of confidence. This held for frozen and fully trained CNNs, for transformers, for a varied zoo of architectures and against small ensembles. Proposition~\ref{prop:redundancy} and the within-decile test of Section~\ref{subsec:mechanism} explain why. Whether a mild operator flips a prediction is decided mostly by the clean margin, and at a fixed margin the size of the displacement says almost nothing about correctness. A practitioner who wants to rank predictions by trustworthiness should therefore use MSP, the margin or the entropy. TTA mean confidence is a further option if a few extra forward passes are affordable.}

\rev{The exceptions follow the same logic. Proposition~\ref{prop:redundancy} fails when the clean output is itself corrupted, so that two images with the same clean logits differ in how much of that output is real. Additive noise and adversarial perturbations both raise confidence in the wrong class, and a smoothing operator partly undoes them. In those settings consistency carried information that confidence lacked, the size of the gain followed the degree of over-confidence, and a pre-specified prediction was confirmed. DCF is therefore best applied as a check on inputs whose confidence is known to be unreliable, such as noise corruptions and adversarial perturbations. It is preferably computed in feature space rather than from predicted labels. Its limits against adaptive attackers are stated in the results.}

\subsection{What CAG measures}

\rev{Rewriting CAG clarifies its behaviour. Because $\mathrm{CAG}=\mathrm{Acc}-\bar{C}=(1-\bar{C})-(1-\mathrm{Acc})$, the gap is the mean instability minus the error rate. For the models in Table~\ref{tab:reference}, mean instability $1-\bar{C}$ was 0.01 to 0.09, so CAG is negative whenever the error rate exceeds that, which is almost always. A large negative CAG therefore says that many predictions are wrong but stable. Section~\ref{subsec:octres} shows one common reason: a class-prior mismatch that makes the model confidently and consistently predict the majority classes. CAG is thus best read as a prompt to check class priors and per-class recall, and not as a summary of reliability. It does not predict accuracy loss under shift, and its regime labels depend on an arbitrary $\delta$.}

\subsection{Scope and limitations}
\label{subsec:limits}

\rev{The conclusions apply to a defined regime. Most experiments use images at CIFAR-10 to STL-10 scale, and ImageNet-scale backbones serve as frozen feature extractors. The six operators are mild by design, and stronger settings made CS worse (Section~\ref{subsec:shift}). Label preservation under these operators was screened automatically with a vision-language model on three of the four datasets. It was not validated by human raters, and the study does not claim that every perturbed image preserves its label.}

\rev{Several conditions limit how far the model comparisons extend. The natively trained models and the zoo backbones were trained once by others, so for them the seeds vary the sampling and the noise but not the training. The ensembles are last-layer ensembles on shared features, which are less diverse than independently trained deep ensembles. The adaptive attack uses one objective with 20 steps. Stronger or differently weighted attacks may break the trade-off observed here, and no claim of adversarial robustness is made. The conformal certificate gives a marginal guarantee under exchangeability. On OCTMNIST the calibration images come from the validation split, whose class prior differs from the balanced test split, so the guarantee there is not formally covered, even though no violations occurred.}

\rev{The corruption and OCTMNIST analyses have their own scope. The zoo analysis uses self-generated corruptions and not the official CIFAR-10-C files. Flip rates are measured for the six operators, and ImageNet-P, which exists only for ImageNet, is not used. Prior correction recovers part of the OCTMNIST gap, and the remainder is described and not explained. The in-domain control is a network trained from scratch, not a backbone pretrained on retinal images such as RETFound~\cite{zhou2023retfound}.}

\section{Conclusion}
\label{sec:conclusion}

\rev{We asked when the decision consistency of an image classifier, measured under six mild perturbations, says something about the correctness of a prediction that the clean softmax output does not.}

\rev{On clean data the answer was negative. Across 201 evaluations, the consistency score was below MSP for misclassification detection in every model family. It did not add information when combined with MSP, and it cost 8 to 15 times more compute. An exact inequality explains this redundancy: a label flip requires the clean logit margin to be at most the square root of two times the norm of the logit displacement. The same inequality yields a label-free conformal certificate, whose guarantee held in every run.}

\rev{The negative answer has two exceptions. Consistency did add information where confidence is misled, namely under noise corruptions and adversarial perturbations. The gain did not transfer to unseen kinds of shift, and an adaptive attacker can trade attack success for stability.}

\rev{A separate finding concerns the consistency-accuracy gap. The gap is large and negative on OCTMNIST, and backbone mismatch alone does not explain it. Part of it is a class-prior effect. Prior correction recovers about 40\% of the gap, and inducing label shift on CIFAR-10 reproduces the pattern. The protocol provides a baseline-controlled way to evaluate future perturbation-based reliability scores.}

\section*{CRediT authorship contribution statement}
\textbf{Amit Kalita:} Conceptualization, Methodology, Software, Investigation, Formal analysis, Visualization, Writing \rev{(original draft)}. \textbf{Moirangthem Tiken Singh:} Conceptualization, Methodology, Supervision, Validation, Writing \rev{(review and editing)}.

\section*{Declaration of Competing Interest}
The authors declare that they have no known competing financial interests or personal relationships that could have appeared to influence the work reported in this paper.

\section*{Declaration of Generative AI and AI Assisted Technologies in the Writing Process}
During the preparation of this work the authors used generative artificial intelligence tools to improve language clarity and to help with analysis code. After using these tools, the authors reviewed and edited the content and code as needed and take full responsibility for the content of the publication.

\section*{Data Availability}
All datasets used in this study are publicly available from their original creators: CIFAR-10, CIFAR-10-C, CIFAR-100, STL-10, SVHN, the Intel Image Classification dataset and OCTMNIST (MedMNIST v2).

\section*{Code Availability}
The code for every experiment, the logged clean and perturbed logits of every run, the per-sample tables and the operator audit files are available from the corresponding author and will be released in a public repository upon publication.

\bibliographystyle{IEEEtran}
\bibliography{ref}

\end{document}